\documentclass[11pt, a4paper]{lumia}

\usepackage[sort&compress]{natbib}
\usepackage{xspace}

\theoremstyle{plain}
\newtheorem{theorem}{Theorem}
\newtheorem{proposition}[theorem]{Proposition}
\newtheorem*{proposition*}{Proposition}

\theoremstyle{definition}

\theoremstyle{definition}

\def\eqref#1{equation~\ref{#1}}

\usepackage{graphicx}           
\usepackage{tikz}               
\usepackage[edges]{forest}      

\usepackage{url}                
\usepackage{xurl}               

\usepackage{array}              
\usepackage{longtable}          
\usepackage{multirow}           
\usepackage{makecell}           
\usepackage{ragged2e}           

\usepackage{mathtools}          
\usepackage{nicefrac}           

\usepackage{algorithm}          
\usepackage{algorithmicx}       
\usepackage{algpseudocode}      
\usepackage{listings}           

\usepackage{subcaption}         
\usepackage{wrapfig}            
\usepackage[export]{adjustbox}  

\usepackage{changes}            
\usepackage{xspace}             
\usepackage[normalem]{ulem}     
\usepackage{CJKutf8}            

\usepackage[tikz]{bclogo}       
\usepackage[framemethod=tikz]{mdframed} 

\usepackage{lipsum}             
\usepackage{tocloft}            
\usepackage{afterpage}          
\usepackage{bbding}             
\usepackage{epigraph}           
\usepackage{minitoc}            
\usepackage{multicol}           
\usepackage{textgreek}          

\newcolumntype{P}[1]{>{\RaggedRight\arraybackslash}p{#1}}

\definecolor{uclablue}{RGB}{39, 116, 174}
\definecolor{bigaired}{RGB}{156, 0, 0}
\definecolor{myblue}{HTML}{598BE7}
\definecolor{mildblue}{RGB}{31,119,180}
\definecolor{sectionblue}{RGB}{70, 130, 180}
\definecolor{methodblue}{RGB}{0, 150, 136}
\definecolor{bgblue}{RGB}{245,243,253}
\definecolor{ttblue}{RGB}{91,194,224}
\definecolor{mygreen}{rgb}{0.64, 0.56, 0.88}
\definecolor{myyellow}{rgb}{0.68, 0.6, 0.1}
\definecolor{fancygreen}{rgb}{0.33, 0.68, 0.20}
\definecolor{salmon}{rgb}{0.94, 0.52, 0.49}
\definecolor{tablegreen}{rgb}{0.82, 0.94, 0.75}
\definecolor{tableblue}{rgb}{0.81, 0.90, 0.94}
\definecolor{tablered}{rgb}{0.97, 0.85, 0.85}
\definecolor{tableorange}{rgb}{0.96, 0.85, 0.81}
\definecolor{myorange}{rgb}{1.0, 0.49, 0.0}
\definecolor{tlgreen}{rgb}{0.33, 0.68, 0.20}
\definecolor{darkgreen}{RGB}{0,100,0}
\definecolor{darkred}{RGB}{200, 0, 0}
\definecolor{customyellow}{HTML}{FFFACD}
\definecolor{refinegreen}{RGB}{0, 128, 75}
\definecolor{scoregreen}{RGB}{34, 139, 34}
\definecolor{hidden-blue}{RGB}{194,232,247}
\definecolor{hidden-black}{RGB}{20,68,106}
\definecolor{yes}{HTML}{C6EFCE}
\definecolor{no}{HTML}{FFC7CE}
\definecolor{partial}{HTML}{FFEB9C}
\definecolor{external}{HTML}{D9E1F2}
\definecolor{hdr}{HTML}{F2F2F2}
\definecolor{GRPOrow}{gray}{0.96}
\definecolor{FlowRLrow}{RGB}{225,236,255}
\definecolor{FlowBlue}{RGB}{80,120,210}
\definecolor{GRPOGray}{gray}{0.35}
\definecolor{mygrey}{RGB}{200,200,200}

\setlist[itemize]{leftmargin=20pt, noitemsep, topsep=0pt}

\NewDocumentCommand{\kaiyan}{mO{}}{\textcolor{purple}{\textsuperscript{\textit{kaiyan}}\textsf{\textbf{\small[#1]}}}}
\NewDocumentCommand{\yuxin}{mO{}}{\textcolor{cyan}{\textsuperscript{\textit{yuxin}}\textsf{\textbf{\small[#1]}}}}
\NewDocumentCommand{\bx}{mO{}}{\textcolor{green}{\textsuperscript{\textit{bx}}\textsf{\textbf{\small[#1]}}}}
\NewDocumentCommand{\at}{mO{}}{\textcolor{red}{\textsuperscript{\textit{AT}}\textsf{\textbf{\small[#1]}}}}
\NewDocumentCommand{\re}{mO{}}{\textcolor{blue}{\textsuperscript{\textit{RE}}\textsf{\textbf{\small[#1]}}}}
\NewDocumentCommand{\ybsun}{mO{}}{\textcolor{magenta}{\textsuperscript{\textit{youbang}}\textsf{\textbf{\small[#1]}}}}
\NewDocumentCommand{\runze}{mO{}}{\textcolor{orange}{\textsuperscript{\textit{runze}}\textsf{\textbf{\small[#1]}}}}
\NewDocumentCommand{\add}{mO{}}{\textcolor{darkgreen}{\textsuperscript{\textit{Maybe Consider Discuss}}\textsf{\textbf{[#1]}}}}

\newcommand{\cmark}{\textcolor{darkgreen}{\boldmath$\checkmark$}}
\newcommand{\xmark}{\textcolor{darkred}{\boldmath$\times$}}

\newcommand{\modelicon}[2]{\raisebox{-0.2\height}{\includegraphics[height=#1]{img/#2}}}

\DeclareMathOperator*{\argmax}{arg\,max}

\newenvironment{itemize*}%
 {\leftmargini=10pt\begin{itemize}%
  \setlength{\itemsep}{0pt}%
  \setlength{\parskip}{0pt}%
  }%
 {\end{itemize}}

\newenvironment{enumerate*}%
 {\begin{enumerate}%
  \setlength{\itemsep}{0pt}%
  \setlength{\parskip}{0pt}}%
 {\end{enumerate}}

\newcommand{\cellstatus}[1]{%
  \begingroup
  \StrTrim{#1}[\statusval]%
  \IfStrEq{\statusval}{Yes}{\cellcolor{yes}\cmark}{}%
  \IfStrEq{\statusval}{No}{\cellcolor{no}\xmark}{}%
  \IfBeginWith{\statusval}{Yes (}{\cellcolor{yes}\cmark~\textit{\statusval\unskip}}{}%
  \IfStrEq{\statusval}{Partial}{\cellcolor{partial}\textbf{Partial}}{}%
  \IfStrEq{\statusval}{External}{\cellcolor{external}\textbf{External}}{}%
  \endgroup
}

\newtcolorbox{myboxi}[1][]{
  breakable,
  title=#1,
  colback=red!5,
  colbacktitle=red!5,
  coltitle=black,
  fonttitle=\bfseries,
  bottomrule=0pt,
  toprule=0pt,
  leftrule=2pt,
  rightrule=2pt,
  titlerule=0pt,
  arc=0pt,
  outer arc=0pt,
  colframe=red,
}

\newtcolorbox{myboxnote}[1][]{
  breakable,
  title=#1,
  colback=orange!0,
  colbacktitle=orange!0,
  coltitle=black,
  fonttitle=\bfseries,
  bottomrule=0pt,
  toprule=0pt,
  leftrule=2pt,
  rightrule=2pt,
  titlerule=0pt,
  arc=0pt,
  outer arc=0pt,
  colframe=orange,
}

\newtcolorbox{myboxii}[1][]{
  breakable,
  freelance,
  title=#1,
  colback=white,
  colbacktitle=white,
  coltitle=black,
  fonttitle=\bfseries,
  bottomrule=0pt,
  boxrule=0pt,
  colframe=white,
  overlay unbroken and first={
  \draw[red!75!black,line width=3pt]
    ([xshift=5pt]frame.north west) -- 
    (frame.north west) -- 
    (frame.south west);
  \draw[red!75!black,line width=3pt]
    ([xshift=-5pt]frame.north east) -- 
    (frame.north east) -- 
    (frame.south east);
  },
  overlay unbroken app={
  \draw[red!75!black,line width=3pt,line cap=rect]
    (frame.south west) -- 
    ([xshift=5pt]frame.south west);
  \draw[red!75!black,line width=3pt,line cap=rect]
    (frame.south east) -- 
    ([xshift=-5pt]frame.south east);
  },
  overlay middle and last={
  \draw[red!75!black,line width=3pt]
    (frame.north west) -- 
    (frame.south west);
  \draw[red!75!black,line width=3pt]
    (frame.north east) -- 
    (frame.south east);
  },
  overlay last app={
  \draw[red!75!black,line width=3pt,line cap=rect]
    (frame.south west) --
    ([xshift=5pt]frame.south west);
  \draw[red!75!black,line width=3pt,line cap=rect]
    (frame.south east) --
    ([xshift=-5pt]frame.south east);
  },
}

\tcbset{
  takeawaysbox/.style={
    title=Takeaways,
    colback=lightblue!80,
    colframe=black,
    fonttitle=\bfseries\small,
    coltitle=white,
    colbacktitle=black,
    enhanced,
    attach boxed title to top left={xshift=2.5mm,yshift=-2.5mm},
    boxed title style={rounded corners, size=small, colframe=black, colback=black},
    width=\linewidth,
    arc=3.5mm
  }
}

\mdfdefinestyle{mystyle}{%
  rightline=true,
  innerleftmargin=10,
  innerrightmargin=10,
  outerlinewidth=3pt,
  topline=false,
  rightline=true,
  bottomline=false,
  skipabove=\topsep,
  skipbelow=\topsep
}

\tikzset{%
    every node/.style={font=\tiny},
    parent/.style =          {align=center,text width=2cm,rounded corners=3pt, line width=0.3mm, fill=gray!10,draw=gray!80},
    child/.style =           {align=center,text width=2.0cm,rounded corners=3pt, fill=blue!10,draw=blue!80,line width=0.3mm},
    grandchild/.style =      {align=center,text width=2cm,rounded corners=3pt},
    greatgrandchild/.style = {align=center,text width=1.5cm,rounded corners=3pt},
    greatgrandchild2/.style = {align=center,text width=1.5cm,rounded corners=3pt},    
    referenceblock/.style =  {align=center,text width=1.5cm,rounded corners=2pt},
    pretrain/.style =           {align=center,text width=2.0cm,rounded corners=3pt, fill=blue!10,draw=blue!80,line width=0.3mm},   
    pretrain_work/.style =           {align=center, text width=8.5cm,rounded corners=3pt, fill=blue!10,draw=blue!0,line width=0.3mm},  
    template/.style =           {align=center,text width=2.0cm,rounded corners=3pt, fill=red!10,draw=red!80,line width=0.3mm},   
    template_work/.style =           {align=center,text width=8.5cm,rounded corners=3pt, fill=red!10,draw=red!0,line width=0.3mm},    
    answer/.style =           {align=center,text width=2.0cm,rounded corners=3pt, fill= cyan!10,draw= cyan!80,line width=0.3mm},   
    answer_work/.style =           {align=center,text width=8.5cm,rounded corners=3pt, fill= cyan!10,draw= cyan!0,line width=0.3mm},      
    multiple/.style =           {align=center,text width=2.0cm,rounded corners=3pt, fill= orange!10,draw= orange!80,line width=0.3mm},   
    multiple_work/.style =           {align=center,text width=8.5cm,rounded corners=3pt, fill= orange!10,draw= orange!0,line width=0.3mm},        
    tuning/.style =           {align=center,text width=2.0cm,rounded corners=3pt, fill= magenta!10,draw= magenta!80,line width=0.3mm},   
    tuning_work/.style =           {align=center,text width=8.5cm,rounded corners=3pt, fill= magenta!10,draw= magenta!0,line width=0.3mm},          
}

\newcommand{\lstbg}[3][0pt]{{\fboxsep#1\colorbox{#2}{\strut #3}}}

\lstdefinelanguage{diff}{
  basicstyle=\ttfamily\small,
  morecomment=[f][\lstbg{red!20}]-,
  morecomment=[f][\lstbg{green!20}]+,
}

\lstdefinelanguage{diffpython}{
  language=diff,
  morekeywords={def, if, else, for, while, return, import, from, as, class, with, try, except, finally, raise, lambda, and, or, not, in, is, None, True, False},
  morecomment=[l]{\#},
  morestring=[b]",
  morestring=[b]',
}

\usepackage{xspace}

\usepackage{xcolor}

\definecolor{ForestGreen}{RGB}{34,139,34}
\definecolor{myyellow}{RGB}{181, 181, 27}

\usepackage[table,xcdraw,usenames,dvipsnames]{xcolor}

\usepackage{amsmath}
\usepackage{amssymb}
\usepackage{mathtools}
\usepackage{amsthm}
\usepackage{fontawesome5} 

\usepackage[capitalize,noabbrev]{cleveref}

\usepackage{multirow}
\usepackage{makecell}
\usepackage{enumerate}
\usepackage{enumitem}
\usepackage{pifont}

\usepackage[ruled,algo2e,vlined]{algorithm2e}
\SetKwInOut{Input}{Input}\SetKwInOut{Output}{Output}
\SetKwComment{Comment}{$\triangleright$\ }{}

\definecolor{darkgreen}{RGB}{30, 130, 30}
\definecolor{cream}{RGB}{253, 250, 242}
\renewcommand{\cmark}{\textcolor{darkgreen}{\ding{51}}} 
\renewcommand{\xmark}{\textcolor{red}{\ding{55}}}       

\usepackage{listings}

\usepackage{ulem}
\usepackage{pifont}
\def\method{Orchestra-o1}
\def\model{Orchestra-o1-8B}

\setheadertext{LUMIA Lab}

\correspondingemail{\emailicon\ fzhang25@cse.cuhk.edu.hk \quad $^*$ Equal Contribution \quad $^\dagger$ Corresponding Author}

\githublink{https://github.com/zfkarl/Orchestra-o1} \huggingfacelink{https://huggingface.co/Karl28/Orchestra-o1-8B}

\title{Orchestra-o1: Omnimodal Agent Orchestration}
\setheadertitle{Orchestra-o1: Omnimodal Agent Orchestration}

\author{%
  Fan Zhang$^{1,*}$, Vireo Zhang$^{*}$, Shengju Qian$^{2,\dagger}$, Haoxuan Li$^{3}$, Hao Wu$^{4}$, Jinyang Wu$^{4}$,\linebreak Donghao Zhou$^{1}$, Zhihong Zhu$^{3}$, Zheng Lian$^{5}$, Xin Wang$^{2}$, Pheng-Ann Heng$^{1,\dagger}$\\
  $^1$CUHK \hspace{1mm}
  $^2$LIGHTSPEED \hspace{1mm}
  $^3$PKU \hspace{1mm}
  $^4$THU \hspace{1mm}
  $^5$Tongji University
  \\
}

\begin{document}

\begin{abstract}
The recent success of agent swarms has shifted the paradigm of large language model (LLM)-based agents from single-agent workflows to multi-agent systems, highlighting the importance of agent orchestration for task decomposition and collaboration.
However, existing orchestration frameworks are limited to a narrow set of modalities and struggle to generalize to more complex settings where heterogeneous modalities coexist and interact.
This limitation becomes particularly pronounced in omnimodal scenarios, where tasks require the unified understanding and coordination of diverse inputs such as text, image, audio, and video.
In this work, we propose \method{}, an omnimodal agent orchestration framework designed to support efficient agent collaboration across multiple modalities.
\method{} introduces a unified orchestration mechanism that enables modality-aware task decomposition, online sub-agent specialization, and parallel sub-task execution.
This scalable design allows agent systems to effectively tackle complex real-world tasks involving heterogeneous information sources, surpassing the second-best approach by $10.3\%$ accuracy on the OmniGAIA benchmark.
Furthermore, we introduce decision-aligned group relative policy optimization (DA-GRPO), an efficient agentic reinforcement learning approach for training \model{}, which also achieves state-of-the-art performance against all existing open-source omnimodal agents.
The source code is publicly available at the above links.

\end{abstract}

\maketitle


\begin{figure*}[h]
    \centering
    \includegraphics[width=\linewidth]{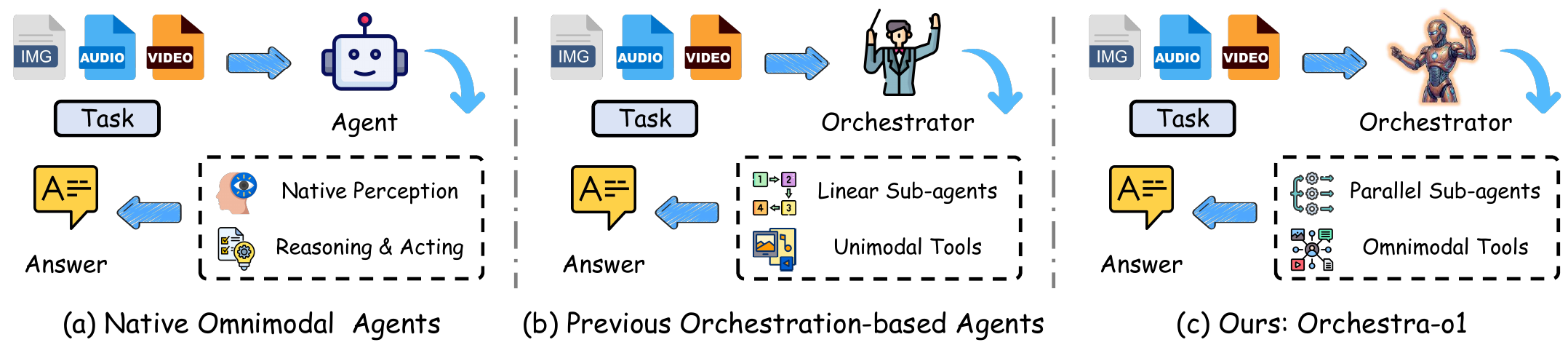}
    \caption{Comparison among three types of omnimodal agents.}
    \label{fig:cmp}
\end{figure*}

\section{Introduction}

Large language model (LLM)-based agents \citep{luo2025large,wang2024survey} have recently emerged as a powerful paradigm for building intelligent systems that can reason, plan, use tools, and interact with external environments. By augmenting LLMs with harness mechanisms \citep{pan2026natural,meng2026agent}, agent systems have substantially expanded the boundary of what language models can accomplish. Representative applications such as code generation and execution \citep{zhang2024codeagent,huang2023agentcoder}, autonomous web research \citep{team2025tongyi,qiao2025webresearcher}, interactive problem solving \citep{yu2026webanchor,tao2025webshaper}, and open-ended computer-use tasks \citep{agashe2025agent,wangcomputer} have demonstrated the potential of LLM agents to reshape human productivity and information access. More recently, the success of agent swarms \citep{team2026kimi} has further shifted the research focus from single-agent workflows to multi-agent systems, where a main agent coordinates multiple specialized agents to decompose complex tasks, execute sub-tasks, and aggregate intermediate results. This paradigm highlights the importance of agent orchestration, which determines how agents are created, specialized, scheduled, and coordinated during task solving.

Despite this progress, most existing LLM-based agent systems are still designed for a limited range of modalities, typically focusing on either pure-text tasks \citep{zhang2024cut} or vision-language tasks \citep{geng2025webwatcher}. This creates a clear gap between current agent research and real-world scenarios, where information is inherently omnimodal and often involves the coexistence and interaction of text, image, audio, and video. In everyday situations, humans naturally process heterogeneous sensory signals in a unified manner. For example, during face-to-face communication, people simultaneously interpret spoken language, facial expressions, gestures, and environmental cues, and then decide how to respond accordingly. Such omnimodal understanding and decision-making are natural for humans but remain highly challenging for existing agents. To solve omnimodal tasks, an agent must not only perceive information from diverse modalities, but also reason over their interactions, decide which specialized capabilities are needed, and coordinate actions across multiple tools or sub-agents. This requires a unified framework that supports both omnimodal perception and high-level agentic decision-making.

As shown in Figure \ref{fig:cmp}, current omnimodal agents can be broadly categorized into two types. The first category is \textit{native omnimodal agents} \citep{team2026qwen3}, which directly employ an omnimodal large language model (OLLM) as the agentic backend and equip it with various action tools. In this design, the same model is expected to perform perception, reasoning, planning, and tool-use simultaneously. However, existing OLLMs still exhibit limited capability in jointly handling perception and action, especially when tasks require long-horizon reasoning, external information seeking, code execution, or fine-grained cross-modal understanding. As a result, even strong proprietary omnimodal models such as Gemini-3-Pro \citep{gemini3pro} achieve only $62.5\%$ accuracy on the challenging benchmark OmniGAIA \citep{li2026omnigaia}. The second category is \textit{orchestration-based agents} \citep{ruan2026aorchestra}, which decouple perception and action from high-level reasoning. In such systems, a text-based language model usually serves as the main agent or orchestrator, while perception and action are delegated to specialized sub-agents equipped with corresponding tools. This design separates high-level decision-making from low-level modality processing, making the system more modular, extensible, and potentially more scalable for complex omnimodal tasks.

In this paper, we focus on orchestration-based omnimodal agents. Designing an effective omnimodal agent swarm, however, is non-trivial for the following reasons. First, many powerful closed-source agent swarm frameworks, such as Kimi \citep{team2026kimi} and Claude \citep{claudeopus46}, are hidden behind proprietary APIs, making it difficult to extend them for omnimodal research. Second, existing open-source agent orchestration frameworks \citep{ruan2026aorchestra,su2025toolorchestra} are often limited by incomplete perception and action toolsets, as well as relatively rigid and linear sub-agent workflows. These limitations restrict both the scalability and efficiency of agent systems when handling complex tasks involving heterogeneous modalities. Towards this end, we propose \method{}, an omnimodal agent orchestration framework designed to support efficient agent collaboration across multiple modalities. At the model level, \method{} supports flexible agentic backends, allowing both the main agent and sub-agents to be instantiated with different models, including open-source models and proprietary models. At the tool level, we provide a unified tool ecosystem consisting of perception tools and action tools, enabling the system to understand and coordinate diverse inputs such as text, image, audio, and video, while also supporting external information seeking and code execution. At the scaffold level, \method{} introduces a collaborative orchestration mechanism based on agent skills and context memory, enabling modality-aware task decomposition, online sub-agent specialization, and parallel sub-task execution. Together, these designs make \method{} both effective and efficient for solving complex omnimodal agent tasks. When using GPT-5 \citep{openaigpt5} as the main agent, \method{} establishes a new state-of-the-art (SOTA) on the OmniGAIA benchmark and substantially outperforms competing baselines, achieving a $32.8\%$ improvement over AOrchestra \citep{ruan2026aorchestra} and a $10.3\%$ improvement over Gemini-3-Pro \citep{gemini3pro}.

In addition to the orchestration framework, we further explore how to train an open-source model to serve as the main agent in \method{}. To this end, we propose \textit{decision-aligned group relative policy optimization} (DA-GRPO), an efficient offline agentic reinforcement learning algorithm for enhancing orchestration decision-making. DA-GRPO extends GRPO \citep{guo2025deepseek} with a design specifically tailored for agent orchestration. Unlike the original GRPO, which focuses solely on final-answer correctness, DA-GRPO explicitly aligns the main agent's step-level decisions with high-quality reference trajectories, covering key decisions such as task delegation, sub-agent selection, tool usage, and answer generation. Leveraging high-quality synthetic trajectories and a multi-dimensional rubric-based reward design, we train \model{} based on Qwen3-8B \citep{yang2025qwen3} to serve as an open-source main agent within the \method{} framework. Experimental results demonstrate that \model{} significantly improves the performance of open-source omnimodal agents on OmniGAIA, increasing the previous best accuracy from $20.8\%$ to $30.0\%$.

In summary, the main contributions of this paper are as follows:
\begin{itemize}[leftmargin=*]
    \item \textit{\textbf{Omnimodal Agent Orchestration Framework.}}
    We propose \method{}, an omnimodal agent orchestration framework for complex real-world agent tasks. Through modality-aware task decomposition, online sub-agent specialization, and parallel sub-task execution, \method{} decouples high-level orchestration from specialized perception and action execution, serving as a scalable open-source framework for building omnimodal agent swarms.
    \item \textit{\textbf{Efficient Agent Orchestration Training Recipe.}} We develop DA-GRPO, an efficient agentic reinforcement learning algorithm for orchestration training. DA-GRPO aligns the main agent's step-level orchestration decisions with high-quality reference trajectories based on multi-dimensional rubric reward design, enabling open-source models to acquire stronger delegation, planning, and decision-making capabilities in omnimodal agent systems.
    \item \textit{\textbf{Multifaceted Experimental Validation.}} Extensive experiments demonstrate that \method{} significantly outperforms existing omnimodal agents. With a strong proprietary main agent, it achieves a new state-of-the-art on OmniGAIA, surpassing the second-best approach by $10.\%$ accuracy. Compared to AOrchestra, \method{} further achieves faster inference and better cost-effectiveness, benefiting from its parallelizable orchestration design. Moreover, when trained with DA-GRPO, \model{} consistently outperforms existing open-source omnimodal agents by a large margin.
\end{itemize}

\section{Related Work}

\subsection{LLM-based Agent Orchestration}

Recent advances in LLM-based agents have shifted from single-agent reasoning systems to multi-agent orchestration frameworks. Early efforts primarily focus on enhancing tool use and planning capabilities within a single agent \citep{yao2022react,schick2023toolformer}, where the model iteratively interacts with external tools to solve complex tasks. 
More recently, multi-agent systems have emerged as a promising direction, where a central orchestrator coordinates multiple specialized agents to improve scalability and task decomposition. Representative works such as AutoGen-style systems \citep{wu2024autogen} and agent swarms \citep{team2026kimi} demonstrate that dividing responsibilities across agents can significantly improve performance on complex reasoning and interactive tasks.
However, existing orchestration frameworks are mostly designed for text-based or limited vision-language settings \citep{ruan2026aorchestra,zhang2026flowsteer}, and often rely on linear or heuristic-driven workflows. In contrast, real-world tasks require more flexible coordination strategies that can dynamically adapt agent roles, parallelize execution, and integrate heterogeneous tools. Our work differs from prior studies by focusing on a unified orchestration framework that supports modality-aware decomposition and scalable multi-agent collaboration in omnimodal environments.

\subsection{Omnimodal Agent Intelligence}

Omnimodal intelligence extends traditional vision-language or audio-language systems to handle heterogeneous inputs such as text, image, audio, and video within a unified framework. Early multimodal models mainly focus on bimodal settings, such as vision-language understanding \citep{li2023blip,liu2023visual}, which demonstrate strong capabilities in aligning visual and textual representations.
With the development of large-scale multimodal models, recent works have begun exploring omnimodal agents \citep{gemini3pro,team2026qwen3,ai2025ming,team2025longcat}. These models aim to unify perception and reasoning across multiple modalities, enabling more general interaction capabilities. However, their performance remains limited in complex agentic scenarios that require long-horizon reasoning, tool use, and multi-step decision-making.
To address these limitations, recent approaches introduce external tool augmentation or modular decomposition to improve omnimodal reasoning \citep{li2026omnigaia}. Nevertheless, these methods often lack systematic orchestration mechanisms for coordinating multiple specialized components. In contrast, our work focuses on an explicit omnimodal agent orchestration paradigm, where perception, reasoning, and action are decoupled and coordinated through a structured multi-agent system, enabling more scalable and efficient omnimodal intelligence.

\section{Methodology}

In this section, we first review the background of agent orchestration and introduce the necessary preliminaries (Section \ref{sec:pre}). We then present our proposed omnimodal agent orchestration framework, \method{} (Section \ref{sec:framework}), followed by the training recipe for deriving an open-source main agent within the framework (Section \ref{sec:training}).

\begin{figure*}[t]
    \centering
    \includegraphics[width=\linewidth]{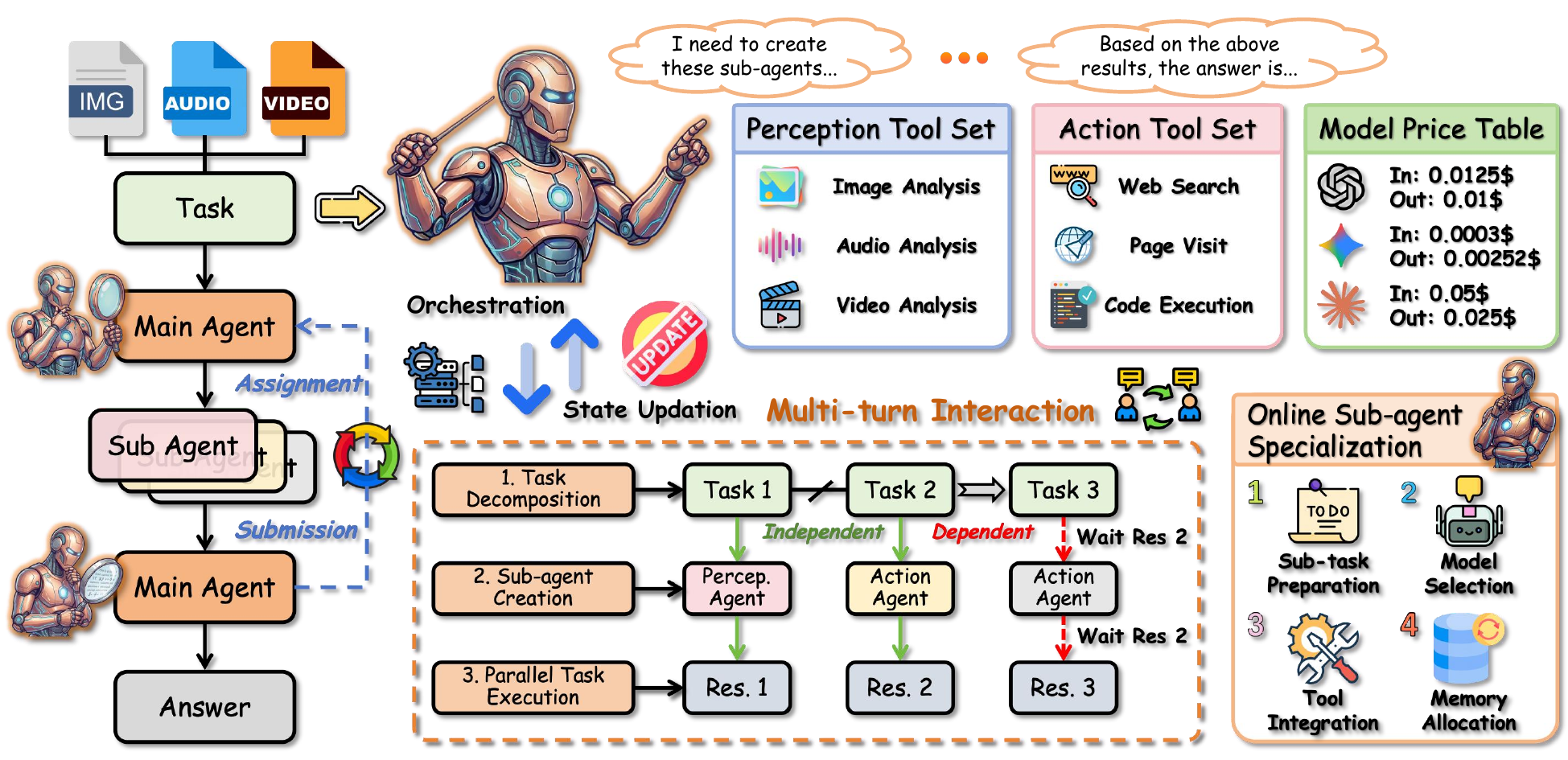}
    \caption{An overview of the \method{} framework.}
    \label{fig:framework}
\end{figure*}

\subsection{Preliminary}\label{sec:pre}

\paragraph{Problem Definition.}
We formulate omnimodal agent orchestration as a multi-round decision-making problem over heterogeneous inputs. Given a task instance $x=(q, \mathcal{M})$, where $q$ denotes the natural-language question and $\mathcal{M}=\{m_i\}_{i=1}^{N}$ denotes a set of auxiliary modality inputs such as images, audios, and videos. The goal is to produce a concise final answer $\hat{a}$ that maximizes the task reward $R(\hat{a}, a^*)$ with respect to the ground-truth answer $a^*$.

\paragraph{System Formulation.}
An orchestration-based agent system consists of a main agent, a set of sub-agent backends, and a tool ecosystem. The main agent $\pi_{\theta}$ acts as an orchestrator rather than directly operating on every modality. At orchestration round $t$, it observes a state:
\begin{equation}
    s_t = \big(q, \mathcal{M}, c_t, H_t, \mathcal{B}, \mathcal{T}\big),
\end{equation}
where $c_t$ is the accumulated context, $H_t$ is the structured sub-task history, $\mathcal{B}$ is the set of available sub-agent models, and $\mathcal{T}$ is the set of tools available to sub-agents. The main agent outputs a structured decision $y_t$ from two action types: $ y_t \in \{\mathtt{delegate}, \mathtt{complete}\}$.
If $y_t=\mathtt{complete}$, the main agent terminates the trajectory and returns $\hat{a}$. If $y_t=\mathtt{delegate}$, it generates a batch of $K_t$ sub-tasks:
\begin{equation}
    \mathcal{U}_t = \{u_{t,j}\}_{j=1}^{K_t}, \quad
    u_{t,j}=(I_{t,j}, C_{t,j}, b_{t,j}, \mathcal{T}_{t,j}),
\end{equation}
where $I_{t,j}$ is a sub-task instruction, $C_{t,j}$ is the context passed from previous rounds, $b_{t,j}\in\mathcal{B}$ is the selected sub-agent backend, and $\mathcal{T}_{t,j}\subseteq\mathcal{T}$ is the assigned tool subset. Each sub-task is executed by an independent sub-agent, producing a result tuple $z_{t,j}$ that contains its status, answer-like result, summary, and execution trace. The results are summarized and appended to $H_{t+1}$, after which the main agent either launches another delegation round or produces the final answer.

This formulation highlights two key requirements for omnimodal orchestration. First, the main agent must make \textit{modality-aware decisions}: it needs to identify which inputs and tools are relevant before dispatching sub-tasks. Second, it must make \textit{dependency-aware scheduling decisions}: independent sub-tasks should be executed in parallel, while dependent sub-tasks should be delayed until prerequisite results become available.

\subsection{The \method{} Framework}\label{sec:framework}

Figure \ref{fig:framework} presents the overall architecture of \method{}. The framework is designed as a hierarchical policy that factorizes complex omnimodal problem solving into high-level orchestration and low-level specialized execution. Let $\mathcal{B}=\{b_{\ell}\}_{\ell=1}^{L}$ denote the candidate sub-agent backends and $\mathcal{T}=\mathcal{T}^{\mathrm{perc}}\cup\mathcal{T}^{\mathrm{act}}$ represent the unified tool set, respectively. In \method{}, the perception tool set $\mathcal{T}^{\mathrm{perc}}$ consists of tools for image analysis, audio analysis, and video analysis. The action tool set $\mathcal{T}^{\mathrm{act}}$ contains tools for web search, page visit, and code execution.
At round $t$, the main agent implements a stochastic orchestration policy:
\begin{equation}
    y_t \sim \pi_{\theta}(\cdot \mid s_t), \quad
    y_t=(a_t,\xi_t), \quad a_t\in\{\mathtt{delegate},\mathtt{complete}\},
\end{equation}
where $a_t$ is the high-level action and $\xi_t$ denotes its structured arguments. The system-level trajectory is therefore $\tau = \big(s_1,y_1,Z_1,s_2,y_2,Z_2,\ldots,s_T,y_T\big)$, where $Z_t$ is the set of sub-agent results returned after a delegation action. The objective of \method{} is to maximize expected task utility under latency and monetary budgets:
\begin{equation}
    \max_{\pi_{\theta}}\; \mathbb{E}_{\tau\sim\pi_{\theta}}\Big[R(\hat{a},a^*)
    -\lambda_{c}\operatorname{Cost}(\tau)-\lambda_{l}\operatorname{Latency}(\tau)\Big].
\end{equation}

\paragraph{Flexible Agentic Backends.}
\method{} supports heterogeneous model backends for both the main agent and sub-agents. Each backend $b\in\mathcal{B}$ is represented by a skill vector and a cost-latency profile $\phi(b)=\big(\phi^{\mathrm{txt}}_b,\phi^{\mathrm{img}}_b,\phi^{\mathrm{aud}}_b, \phi^{\mathrm{vid}}_b,\phi^{\mathrm{code}}_b,\kappa_b,\delta_b\big)$,
where the first five terms encode capability scores, while $\kappa_b$ and $\delta_b$ denote unit cost and expected latency. For a candidate sub-task $u$, the main agent predicts a requirement vector $r(u)=\big(r^{\mathrm{txt}}_u,r^{\mathrm{img}}_u,r^{\mathrm{aud}}_u,r^{\mathrm{vid}}_u,r^{\mathrm{code}}_u\big)\in[0,1]$.
The model assignment can be viewed as maximizing a cost-aware matching score:
\begin{equation}
    b^{*}(u)=\argmax_{b\in\mathcal{B}}\;
    \underbrace{\langle r(u),\phi_b\rangle}_{\text{capability match}}
    -\lambda_{c}\kappa_b\,\ell(u)-\lambda_{l}\delta_b,
\end{equation}
where $\ell(u)$ is the estimated token or step length of the sub-task. This formulation captures the practical backend selection strategy in \method{}: easy extraction and search sub-tasks are routed to cheaper models, while difficult omnimodal reasoning sub-tasks are routed to stronger backends.

\paragraph{Unified Omnimodal Tool Ecosystem.}
The tool assignment problem is also formulated as requirement matching. Let each tool $g\in\mathcal{T}$ have a capability vector
$\psi(g)=\big(\psi^{\mathrm{txt}}_g,\psi^{\mathrm{img}}_g,\psi^{\mathrm{aud}}_g, \psi^{\mathrm{vid}}_g,\psi^{\mathrm{web}}_g,\psi^{\mathrm{code}}_g\big)\in\{0,1\}$.
For a sub-task $u$, the selected tool subset is:
\begin{equation}
    \mathcal{T}^{*}(u)=\{g\in\mathcal{T}: p_{\theta}(g\mid u,s_t)>\gamma\},
\end{equation}
or equivalently the solution of a sparse coverage objective:
\begin{equation}
    \mathcal{T}^{*}(u)=\argmax_{\mathcal{S}\subseteq\mathcal{T}}
    \Big[\langle r_T(u),\sum_{g\in\mathcal{S}}\psi(g)\rangle
    -\lambda_{s}|\mathcal{S}|\Big],
\end{equation}
where $r_T(u)$ denotes the tool-side requirement vector. In this view, image, audio, and video tools supply modality evidence, while web search, page visit, and code execution tools supply external knowledge and computation. The complete low-level executor for sub-task $u$ is thus:
\begin{equation}
    e(u)=\big(b^{*}(u),\mathcal{T}^{*}(u)\big).
\end{equation}

\paragraph{Modality-aware Task Decomposition.}
At each round, the main agent first induces a latent dependency graph over unsolved sub-goals $\mathcal{G}_t=(\mathcal{V}_t,\mathcal{E}_t)$ and $\mathcal{V}_t=\{v_{t,1},\ldots,v_{t,n_t}\}$,
where a directed edge $(v_i,v_j)\in\mathcal{E}_t$ means that $v_j$ depends on the result of $v_i$. Each node is associated with a modality mask $\mu(v)\in\{0,1\}$ and a tool mask $\alpha(v)\in\{0,1\}^{|\mathcal{T}|}$,
where $\mu(v)$ indicates whether text, image, audio, or video evidence is required, and $\alpha(v)$ indicates candidate tools. A sub-goal is executable if all of its predecessors have already been completed. Therefore, the ready set at round $t$ is:
\begin{equation}
    \mathcal{R}_t=\{v\in\mathcal{V}_t\setminus\mathcal{C}_t:
    \operatorname{Pred}(v)\subseteq\mathcal{C}_t\},
\end{equation}
where $\mathcal{C}_t$ denotes completed sub-goals. The main agent selects a parallel batch from this ready set:
\begin{equation}
    \mathcal{P}_t=\argmax_{\mathcal{P}\subseteq\mathcal{R}_t}
    \sum_{v\in\mathcal{P}} U_{\theta}(v\mid s_t)
    \quad \text{s.t.}\quad |\mathcal{P}|\le K_{\max},\; \sum_{v\in\mathcal{P}}\operatorname{cost}(v)\le B_t.
\end{equation}
For each selected node $v_{t,j}\in\mathcal{P}_t$, \method{} materializes a concrete sub-task:
\begin{equation}
    u_{t,j}=\Gamma_{\theta}(v_{t,j},s_t)
    =\big(I_{t,j},C_{t,j},b_{t,j},\mathcal{T}_{t,j}\big),
\end{equation}
where $b_{t,j}=b^{*}(u_{t,j})$ and $\mathcal{T}_{t,j}=\mathcal{T}^{*}(u_{t,j})$.
The delegated action is therefore a structured batch decision:
\begin{equation}
    y_t=\mathtt{delegate}(\mathcal{U}_t),\quad
    \mathcal{U}_t=\{u_{t,j}\}_{j=1}^{K_t}, \quad K_t=|\mathcal{P}_t|.
\end{equation}
This mathematical formulation makes the decomposition strategy explicit: \method{} does not only split the task into text strings, but also predicts dependency structure, modality requirements, tool requirements, and backend assignments.

\paragraph{Parallel Sub-task Execution.}
Each delegated sub-task is executed by an independent ReAct-style \citep{yao2022react} sub-agent. For sub-task $u_{t,j}$, the sub-agent trajectory is:
\begin{equation}
    \zeta_{t,j}=\big\{(\rho_{t,j}^{(\ell)},a_{t,j}^{(\ell)},o_{t,j}^{(\ell)})\big\}_{\ell=1}^{L_{t,j}},
    \quad a_{t,j}^{(\ell)}\in\mathcal{T}_{t,j}\cup\{\mathtt{Finish}\},
\end{equation}
where $\rho_{t,j}^{(\ell)}$ is the reasoning state, $a_{t,j}^{(\ell)}$ is the selected tool action, and $o_{t,j}^{(\ell)}$ is the observation. The final sub-agent output is summarized as:
\begin{equation}
    z_{t,j}=\Omega(\zeta_{t,j})=(\sigma_{t,j},\eta_{t,j},\omega_{t,j},c_{t,j},\delta_{t,j}),
\end{equation}
where $\sigma_{t,j}$ is the execution status, $\eta_{t,j}$ is the answer-like result, $\omega_{t,j}$ is a compact trace summary, $c_{t,j}$ is the cost, and $\delta_{t,j}$ is the latency. Since all sub-tasks in $\mathcal{U}_t$ are conditionally independent given $s_t$ and do not share mutable environment states, their execution factorizes as:
\begin{equation}
    p(Z_t\mid \mathcal{U}_t,s_t)=\prod_{j=1}^{K_t}p(z_{t,j}\mid u_{t,j},s_t), \quad Z_t=\operatorname{AsyncExecute}(u_{t,1},\ldots,u_{t,K_t}).
\end{equation}
This factorization yields a formal latency advantage for parallel orchestration, as stated below. The proof can be found in Section \ref{sec:proof_p1}.
\begin{proposition}[Round-level Latency Advantage]\label{proposition1}
Consider an orchestration round $t$ with $K_t\ge 2$ ready sub-tasks whose execution times are $\delta_{t,1},\ldots,\delta_{t,K_t}>0$. Assume these sub-tasks are conditionally independent given $s_t$, do not share mutable environment states during execution, and the only additional overhead of parallel execution is a nonnegative synchronization cost $\delta^{\mathrm{sync}}_t\ge 0$. If a linear orchestrator executes the sub-tasks sequentially, while \method{} launches them asynchronously and waits for all outputs, then we have:
\begin{equation}
    \operatorname{Latency}^{\mathrm{linear}}(t)=\sum_{j=1}^{K_t}\delta_{t,j}, \quad
    \operatorname{Latency}^{\mathrm{parallel}}(t)=\max_{1\le j\le K_t}\delta_{t,j}+\delta^{\mathrm{sync}}_t.
\end{equation}
Moreover, parallel execution is no slower than linear execution if and only if:
\begin{equation}
    \delta^{\mathrm{sync}}_t\le \sum_{j=1}^{K_t}\delta_{t,j}-\max_{1\le j\le K_t}\delta_{t,j}.
\end{equation}
Under this condition, the round-level speedup satisfies:
\begin{equation}
    1\le S_t
    =\frac{\operatorname{Latency}^{\mathrm{linear}}(t)}{\operatorname{Latency}^{\mathrm{parallel}}(t)}
    =\frac{\sum_{j=1}^{K_t}\delta_{t,j}}
    {\max_{1\le j\le K_t}\delta_{t,j}+\delta^{\mathrm{sync}}_t}
    \le K_t.
\end{equation}
The upper bound $S_t=K_t$ is attainable only when $\delta^{\mathrm{sync}}_t=0$ and all sub-task runtimes are equal.
\end{proposition}


\paragraph{Context Memory and Iterative Refinement.}
After each delegation round, \method{} updates a structured memory that stores the evidence returned by all sub-agents. Let the memory be:
$H_t=\{h_1,\ldots,h_{m_t}\}$ and $h=(I,b,\mathcal{T},\sigma,\eta,\omega)$, the update after round $t$ is:
\begin{equation}
    H_{t+1}=H_t\cup\{\operatorname{Summarize}(u_{t,j},z_{t,j})\}_{j=1}^{K_t}.
\end{equation}
To keep the main-agent context within the token budget $L_{\mathrm{ctx}}$, \method{} constructs a compressed context by solving:
\begin{equation}
    C_{t+1}=\argmax_{C:|C|\le L_{\mathrm{ctx}}}
    \Big[I(C;q)+\sum_{h\in H_{t+1}} w(h) I(C;h)\Big],
\end{equation}
where $I(\cdot;\cdot)$ denotes information relevance and $w(h)$ up-weights successful or recently produced evidence. The next orchestration state and budget are:
\begin{equation}
    s_{t+1}=\big(q,\mathcal{M},C_{t+1},H_{t+1},\mathcal{B},\mathcal{T},B_{t+1}\big), \quad  B_{t+1}=B_t-\sum_{j=1}^{K_t}c_{t,j}.
\end{equation}
The main agent terminates when its evidence sufficiency score exceeds a threshold:
\begin{equation}
    p^{\mathrm{stop}}_{\theta}(s_t)=p_{\theta}(a_t=\mathtt{complete}\mid s_t),
    \quad p^{\mathrm{stop}}_{\theta}(s_t)>\tau_{\mathrm{stop}}.
\end{equation}
The final answer is generated from the compressed evidence state $\hat{a}=A_{\theta}(q,\mathcal{M},C_t,H_t)$.
Otherwise, the main agent refines the dependency graph according to new evidence
$\mathcal{G}_{t+1}=\operatorname{Refine}_{\theta}(\mathcal{G}_{t},H_{t+1})$,
and continues delegation. Overall, \method{} differs from linear orchestration frameworks by explicitly modeling omnimodal agent collaboration as a dependency-aware parallel scheduling process with learnable decomposition, model selection, tool selection, evidence aggregation, and stopping decisions.

\paragraph{Theoretical Advantage over Native Omnimodal Agents.}\label{proposition2}
We next provide an information-theoretic justification for why agent orchestration can be preferable to a native omnimodal agent design in heterogeneous tasks. The proof can be found in Section \ref{sec:proof_p2}.
\begin{proposition}[Information Gain from Omnimodal Orchestration]
Let $Y$ denote the latent task answer and let $\mathcal{M}=(M_1,\ldots,M_R)$ denote $R$ modality sources. A native omnimodal agent compresses all modalities into a single internal evidence variable $E_0=f_0(q,\mathcal{M})$ under a fixed context and computation budget. An orchestration-based system assigns modality-aware sub-tasks to specialized sub-agents and obtains evidence variables $E_{\mathrm{orch}}=(E_1,\ldots,E_R)$, where $E_r=f_r(q,M_r,C_r)$ is produced by a backend/tool pair specialized for modality $M_r$. Suppose that: (i) the main agent aggregates all returned evidence without losing information relevant to $Y$; (ii) the native evidence admits modality-wise components $(E^{0}_1,\ldots,E^{0}_R)$ whose joint information upper-bounds the information retained by $E_0$; and (iii) specialized execution is at least as informative as native processing at every modality step, with a strict gain for at least one modality:
\begin{equation}
    I(Y;E_r\mid q,E_{<r})\ge I(Y;E^{0}_r\mid q,E^{0}_{<r}), \quad r=1,\ldots,R,
\end{equation}
and the inequality is strict for some $r$. Then we have:
\begin{equation}
    I(Y;E_{\mathrm{orch}}\mid q)> I(Y;E_0\mid q).
\end{equation}
Moreover, under Bayes-optimal prediction with log loss, whose minimal risk is $\mathcal{R}_{\log}(E)=H(Y\mid q,E)$, orchestration has strictly smaller expected risk:
\begin{equation}
    \mathcal{R}_{\log}(E_{\mathrm{orch}})<\mathcal{R}_{\log}(E_0).
\end{equation}
\end{proposition}

\paragraph{Framework Summary.}
In summary, \method{} implements omnimodal agent orchestration as a closed-loop decision process that separates high-level planning from specialized perception and action execution. The main agent maintains a structured memory, decomposes the task into dependency-aware sub-goals, selects suitable backends and tools for each sub-task, executes independent sub-tasks in parallel, and iteratively compresses returned evidence until the answer is sufficiently supported. This design makes the system both modular and scalable: new modalities, tools, or sub-agent models can be integrated through the same requirement-matching interface, while the dependency-aware scheduler improves latency whenever multiple independent sub-tasks can be solved concurrently. Algorithm~\ref{alg:framework} summarizes the overall workflow of \method{}.

\begin{algorithm2e}[t]
\caption{Workflow of \method{}}
\label{alg:framework}
\Input{Question $q$, Modality Inputs $\mathcal{M}$, Backend Pool $\mathcal{B}$, Tool Set $\mathcal{T}$, Maximum Rounds $T_{\max}$}
\Output{Final Answer $\hat{a}$}
Initialize $H_1\leftarrow\emptyset$, $C_1\leftarrow\emptyset$, and $s_1\leftarrow(q,\mathcal{M},C_1,H_1,\mathcal{B},\mathcal{T})$\;
\For{$t=1,\ldots,T_{\max}$}{
    Sample orchestration decision $y_t=(a_t,\xi_t)\sim\pi_{\theta}(\cdot\mid s_t)$\;
    \If{$a_t=\mathtt{complete}$}{
        Generate $\hat{a}=A_{\theta}(q,\mathcal{M},C_t,H_t)$\;
        \textbf{return} $\hat{a}$\;
    }
    Induce or refine dependency graph $\mathcal{G}_t=(\mathcal{V}_t,\mathcal{E}_t)$ and compute ready set $\mathcal{R}_t$\;
    Select parallel batch $\mathcal{P}_t\subseteq\mathcal{R}_t$ under dependency and budget constraints\;
    \ForEach{$v_{t,j}\in\mathcal{P}_t$ \textnormal{in parallel}}{
        Materialize sub-task $u_{t,j}=\Gamma_{\theta}(v_{t,j},s_t)$\;
        Assign backend $b_{t,j}=b^{*}(u_{t,j})$ and tools $\mathcal{T}_{t,j}=\mathcal{T}^{*}(u_{t,j})$\;
        Execute sub-agent trajectory $\zeta_{t,j}$ and summarize $z_{t,j}=\Omega(\zeta_{t,j})$\;
    }
    Update memory $H_{t+1}\leftarrow H_t\cup\{\operatorname{Summarize}(u_{t,j},z_{t,j})\}_{j=1}^{K_t}$\;
    Compress context $C_{t+1}$ under $L_{\mathrm{ctx}}$ and update remaining budget\;
    Set $s_{t+1}\leftarrow(q,\mathcal{M},C_{t+1},H_{t+1},\mathcal{B},\mathcal{T})$\;
}
Generate fallback answer $\hat{a}=A_{\theta}(q,\mathcal{M},C_{T_{\max}},H_{T_{\max}})$\;
\textbf{return} $\hat{a}$\;
\end{algorithm2e}

\begin{figure*}[t]
    \centering
    \includegraphics[width=\linewidth]{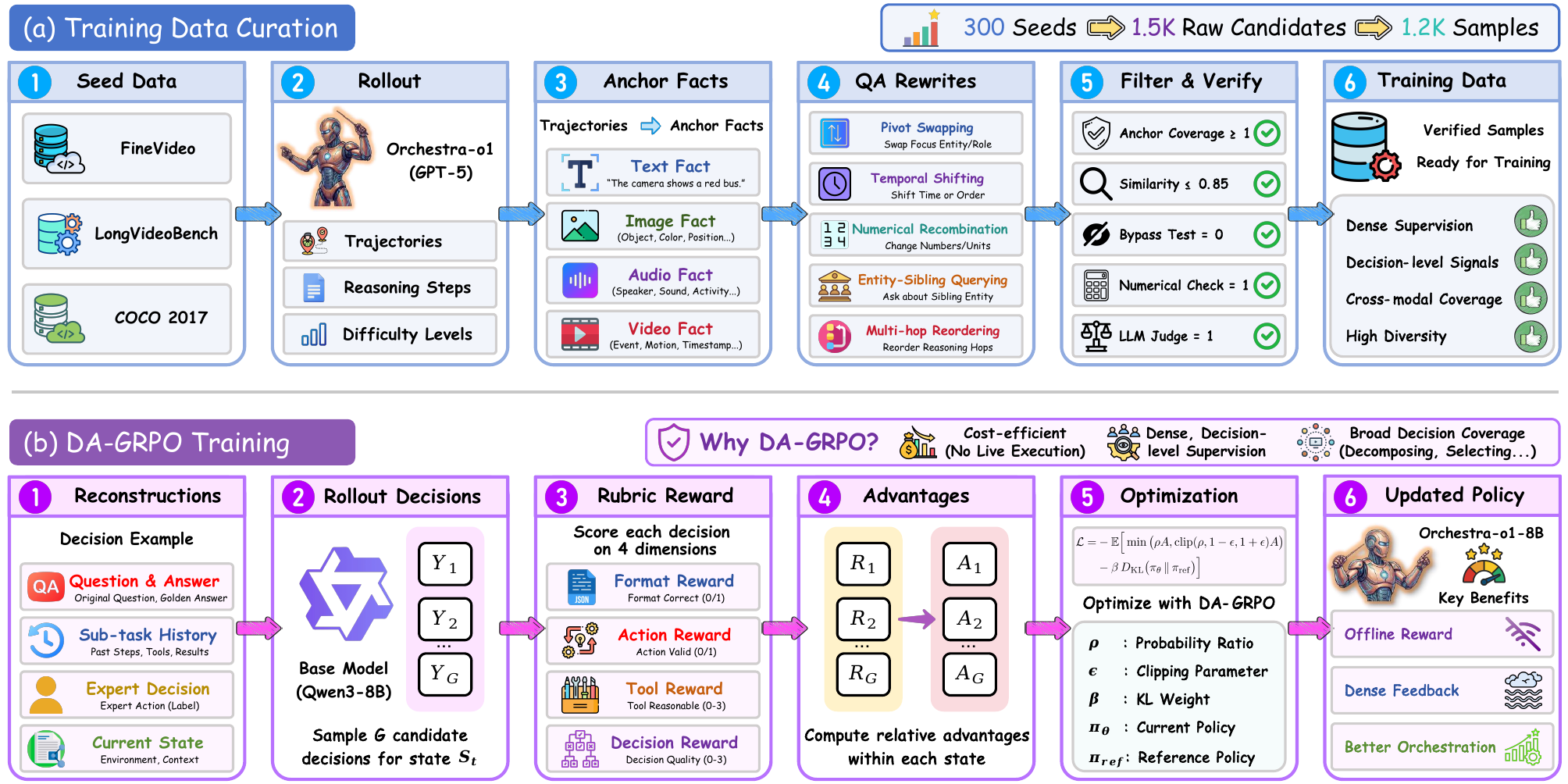}
    \caption{An overview of our training recipe, including (a) data curation pipeline and (b) DA-GRPO training process.}
    \label{fig:dagrpo}
\end{figure*}

\subsection{Training Recipe}\label{sec:training}
Although \method{} can use strong proprietary models as the main agent, a practical open-source agent system requires an open-source model that can make reliable orchestration decisions. We therefore develop a training recipe for deriving \model{} from Qwen3-8B \citep{yang2025qwen3}. The data curation and post-training process are illustrated in Figure \ref{fig:dagrpo}.

\subsubsection{Training Data Curation}

A central challenge in training an open-source orchestrator is the lack of diverse omnimodal tasks with reliable answers and explicit evidence chains. We therefore build a seed-based data curation pipeline on top of public datasets such as FineVideo \citep{Farré2024FineVideo}, LongVideoBench \citep{wu2024longvideobench}, and COCO 2017 \citep{lin2014microsoft}. Given the original seed set:
\begin{equation}
    \mathcal{D}_{0}=\{x_i=(q_i,\mathcal{M}_i,a_i,\mathcal{T}_i)\}_{i=1}^{N},
\end{equation}
where $q_i$ is the question, $\mathcal{M}_i$ denotes the image/audio/video inputs, $a_i$ is the answer, and $\mathcal{T}_i$ is the required tool set, our goal is to create new examples while keeping the original modality files unchanged.
Then we use this seed set to collect successful orchestration trajectories under \method{} (GPT-5 \citep{openaigpt5}) and transform each trajectory into an annotated reasoning solution $r_i$ with a difficulty level $\ell_i$.

The curation pipeline contains three stages. First, we extract modality-grounded anchor facts from the annotated solution and evidence sources. For each seed $x_i$, an LLM extractor produces $A_i=\{(f_{i,m},\mu_{i,m},e_{i,m})\}_{m=1}^{M_i}$, where $f_{i,m}$ is an anchor fact, $\mu_{i,m}$ is its source modality, and $e_{i,m}$ records the supporting step in the annotated solution. These anchors identify the non-bypassable perceptual facts that every valid rewrite must preserve. Second, conditioned on $(x_i,A_i)$, we generate $K_i$ candidate rewrites using five strategy families: pivot swapping, temporal shifting, numerical recombination, entity-sibling querying, and multi-hop reordering. Formally, each candidate is sampled as:
\begin{equation}
    \tilde{x}_{i,k}\sim G_{\omega}(\cdot\mid x_i,A_i,z_{i,k}),
    \quad z_{i,k}\in\{\mathtt{A},\mathtt{B},\mathtt{C},\mathtt{D},\mathtt{E}\},
\end{equation}
subject to the invariants $\tilde{\mathcal{M}}_{i,k}=\mathcal{M}_i$, $\tilde{\ell}_{i,k}=\ell_i$, $\tilde{\mathcal{T}}_{i,k}\subseteq\mathcal{T}_i$, and $|\tilde{s}_{i,k}-s_i|\le 2$, where $\tilde{s}_{i,k}$ and $s_i$ denote the total reasoning steps of the rewritten and original examples. In practice, easy seeds mainly use pivot swapping and entity-sibling querying, medium seeds additionally use temporal or numerical variants, and hard seeds emphasize numerical recombination and multi-hop reordering.
Third, we verify each candidate with a cascade of quality gates. Let $V(\tilde{x}_{i,k})\in\{0,1\}$ denote the final verification decision. We define:
\begin{equation}
\begin{aligned}
V(\tilde{x}_{i,k}) ={}& \mathbb{I}\{\operatorname{AnchorCov}(\tilde{q}_{i,k},A_i)\ge 1\}
\cdot \mathbb{I}\{\operatorname{Sim}(q_i,\tilde{q}_{i,k})\le 0.85\} \cdot \mathbb{I}\{\operatorname{Bypass}(\tilde{q}_{i,k},\tilde{a}_{i,k})=0\}\\
& \cdot \mathbb{I}\{\operatorname{NumCheck}(\tilde{r}_{i,k},\tilde{a}_{i,k})=1\} \cdot \mathbb{I}\{\operatorname{Judge}(\tilde{x}_{i,k})=1\}.
\end{aligned}
\end{equation}
The first two gates enforce anchor coverage and remove near-duplicates by normalized lexical similarity. The third gate performs a modal-bypass test by asking a strong language model to answer without access to $\mathcal{M}_i$: if the answer can still be recovered, the candidate is rejected. The fourth gate executes numerical solutions in a restricted Python sandbox when code execution or numeric answers are involved. The last gate uses an LLM judge to check factual consistency, difficulty preservation, and peer-level duplication among rewrites from the same seed.
All LLM-based processors in our curation pipeline, including the anchor extractor, question rewriter, and verification judge, are implemented with Claude-Opus-4.6 \citep{claudeopus46}.

Finally, all verified rewrites are merged to form the task set $\mathcal{D} =\{\tilde{x}_{i,k}:V(\tilde{x}_{i,k})=1\}$.
Our implementation extracts valid anchors for 300 seeds, generates about 1500 raw rewrite candidates, and retains around 1200 verified examples after filtering. For a trajectory $\tau = \big\{(s_t, y_t^{*}, Z_t)\big\}_{t=1}^{N}$, we create $N$ decision-level examples, where $s_t$ reconstructs the exact main-agent state before the expert decision and $y_t^{*}$ stores the reference orchestration action. This gives dense supervision for delegation, tool assignment, backend selection, parallel scheduling, and stopping decisions.

\subsubsection{Decision-aligned Group Relative Policy Optimization}

We propose decision-aligned group relative policy optimization (DA-GRPO), a GRPO-style training objective tailored for main-agent orchestration. Standard GRPO \citep{guo2025deepseek} samples a group of responses for the same prompt and normalizes rewards within the group to form relative advantages. However, for agent orchestration, final-answer reward is sparse and expensive because it requires executing the whole multi-agent system. DA-GRPO instead evaluates each sampled main-agent decision directly at the current orchestration state, using expert trajectories and a rubric reward that measures whether the decision is well-formed, valid, tool-aware, and strategically useful.

For each prompt $s_i$, the policy samples a group of $G$ candidate decisions $\{y_{i,j}\}_{j=1}^{G}$. Each decision is scored by a multi-dimensional reward:
\begin{equation}
\begin{aligned}
    r_{i,j} ={}& \alpha_{1}\, r^{\mathrm{format}}_{i,j}
              + \alpha_{2}\, r^{\mathrm{action}}_{i,j}
              + \alpha_{3}\, r^{\mathrm{tool}}_{i,j}
              + \alpha_{4}\, r^{\mathrm{decision}}_{i,j},
\end{aligned}
\end{equation}
where $r^{\mathrm{format}}$ measures whether the output is a valid JSON decision, $r^{\mathrm{action}}$ measures whether the action is valid with appropriate parameters, $r^{\mathrm{tool}}$ measures whether the selected tools and sub-task assignments are reasonable, and $r^{\mathrm{decision}}$ measures the overall orchestration decision quality. The first two dimensions are binary, while the latter two are graded and normalized to $[0,1]$. In our implementation, Claude-Haiku-4.5 \citep{claudehaiku45} serves as a lightweight reward model and scores all four dimensions in a single call. The judge is given the current question, ground-truth answer, sub-task history, expert decision, and model output. Importantly, the expert decision is used as a reference rather than the only correct answer: alternative decompositions are rewarded if they are reasonable and likely to solve the task, while a \texttt{complete} decision receives the highest decision-quality score when its answer matches the ground truth. The coefficients for each reward term are empirically set as $\alpha_{1}=\alpha_{2}=0.1$, $\alpha_{3}=0.2$, and $\alpha_{4}=0.6$, prioritizing the tool reward and decision reward, since the two format-related rewards exhibit relatively good initial values.

Given group rewards, DA-GRPO computes the relative advantage of each sampled decision by normalizing within the group:
\begin{equation}
    \hat{A}_{i,j}=\frac{r_{i,j}-\mathrm{Mean}(\{r_{i,k}\}_{k=1}^{G})}
    {\mathrm{Std}(\{r_{i,k}\}_{k=1}^{G})+\epsilon}.
\end{equation}
The policy is then optimized with a clipped policy-gradient objective and a KL regularizer to the reference model:
\begin{equation}
\begin{aligned}
\mathcal{L}_{\mathrm{DA\text{-}GRPO}}(\theta)
= - \mathbb{E}_{i,j}\Big[
\min\big(\rho_{i,j}(\theta)\hat{A}_{i,j},
\mathrm{clip}(\rho_{i,j}(\theta),1-\epsilon,1+\epsilon)\hat{A}_{i,j}\big) - \beta\, D_{\mathrm{KL}}\big(\pi_{\theta}(\cdot|s_i)\,\|\,\pi_{\mathrm{ref}}(\cdot|s_i)\big)\Big],
\end{aligned}
\end{equation}
where $\rho_{i,j}(\theta)=\pi_{\theta}(y_{i,j}|s_i)/\pi_{\mathrm{old}}(y_{i,j}|s_i)$ and $\pi_{\mathrm{ref}}$ is the reference model. This objective encourages the open-source main agent to prefer decisions that are not only syntactically valid but also strategically aligned with successful orchestration behavior.
Compared with outcome-only reinforcement learning, DA-GRPO offers two advantages. First, it avoids repeatedly executing expensive sub-agent trajectories during training, since each decision can be scored offline from the reconstructed state. Second, it provides dense feedback on the main agent's core responsibilities: decomposing tasks, selecting tools, scheduling parallel sub-tasks, and deciding when evidence is sufficient for final answering. We train \model{} with this recipe and deploy it as the main agent in \method{}.

\section{Experiments}

\begin{table*}[t]
\centering
\caption{Category-wise accuracy (\%) on OmniGAIA. The non-orchestration-based models are implemented under the standard ReAct framework. The highest value in each category within each model group is highlighted in bold.}
\resizebox{\textwidth}{!}{%
\label{tab:category_breakdown}
\begin{tabular}{lcccccccccc}
\toprule
\multirow{2.5}{*}{\textbf{Method}} & \multicolumn{9}{c}{\textbf{Category-Wise Breakdown}} & \multirow{2.5}{*}{\textbf{Overall}} \\
\cmidrule(lr){2-10}
& \textbf{Geo.} & \textbf{Tech.} & \textbf{Hist.} & \textbf{Fin.} & \textbf{Sport} & \textbf{Art} & \textbf{Movie} & \textbf{Sci.} & \textbf{Food} & \\
\midrule
\rowcolor{black!5}\multicolumn{11}{c}{\textit{\textbf{Open-Source Agentic Models}}} \\
\midrule
\modelicon{1.0em}{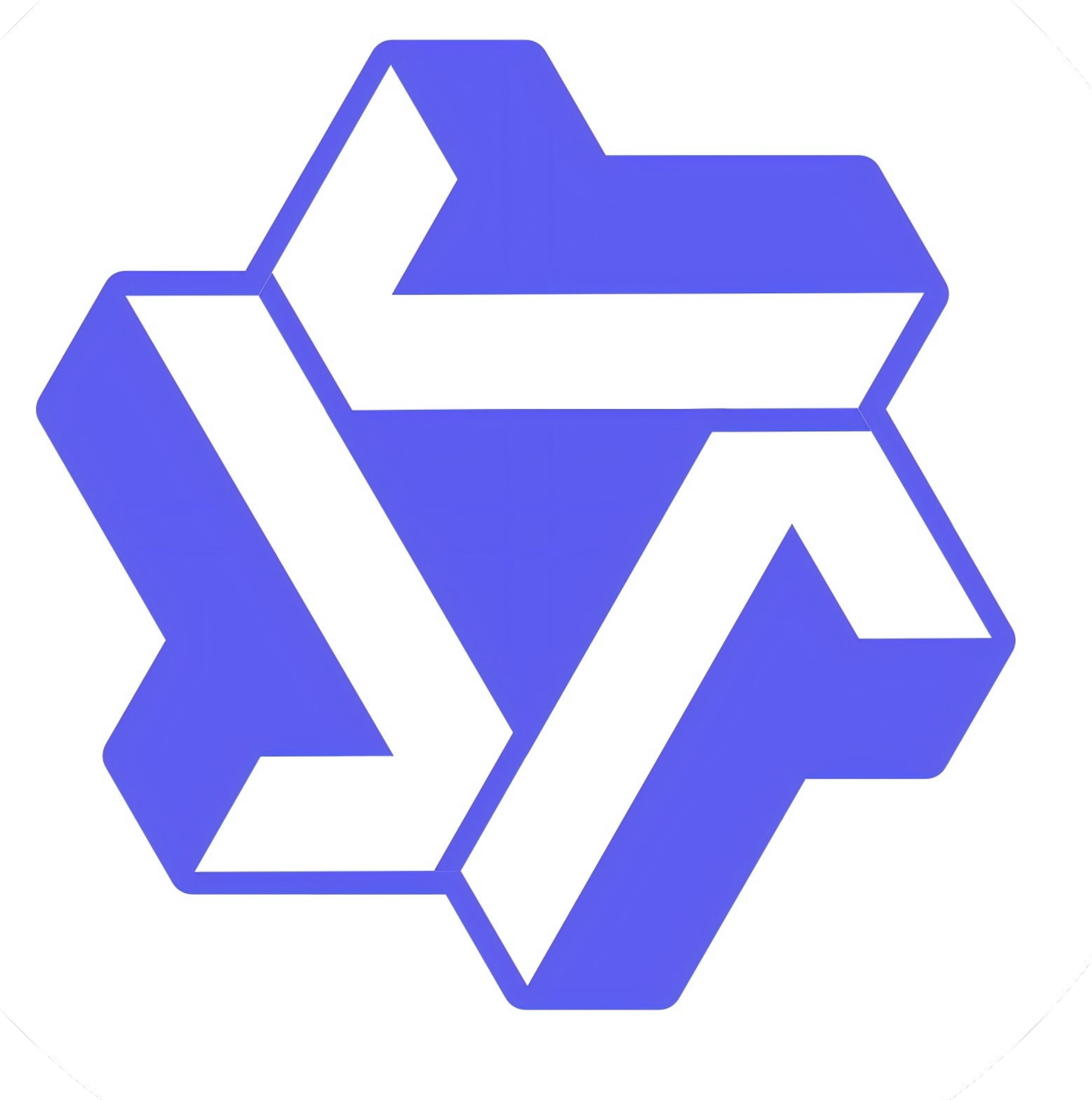}~Qwen2.5-Omni-3B & 0.0 & 2.0 & 4.5 & 0.0 & 0.0 & 0.0 & 0.0 & 3.9 & 0.0 & 1.4 \\
\modelicon{1.0em}{qwen.pdf}~Qwen2.5-Omni-7B & 1.5 & 4.1 & 7.5 & 4.0 & 0.0 & 2.8 & 0.0 & 7.7 & 5.6 & 3.6 \\
\modelicon{0.9em}{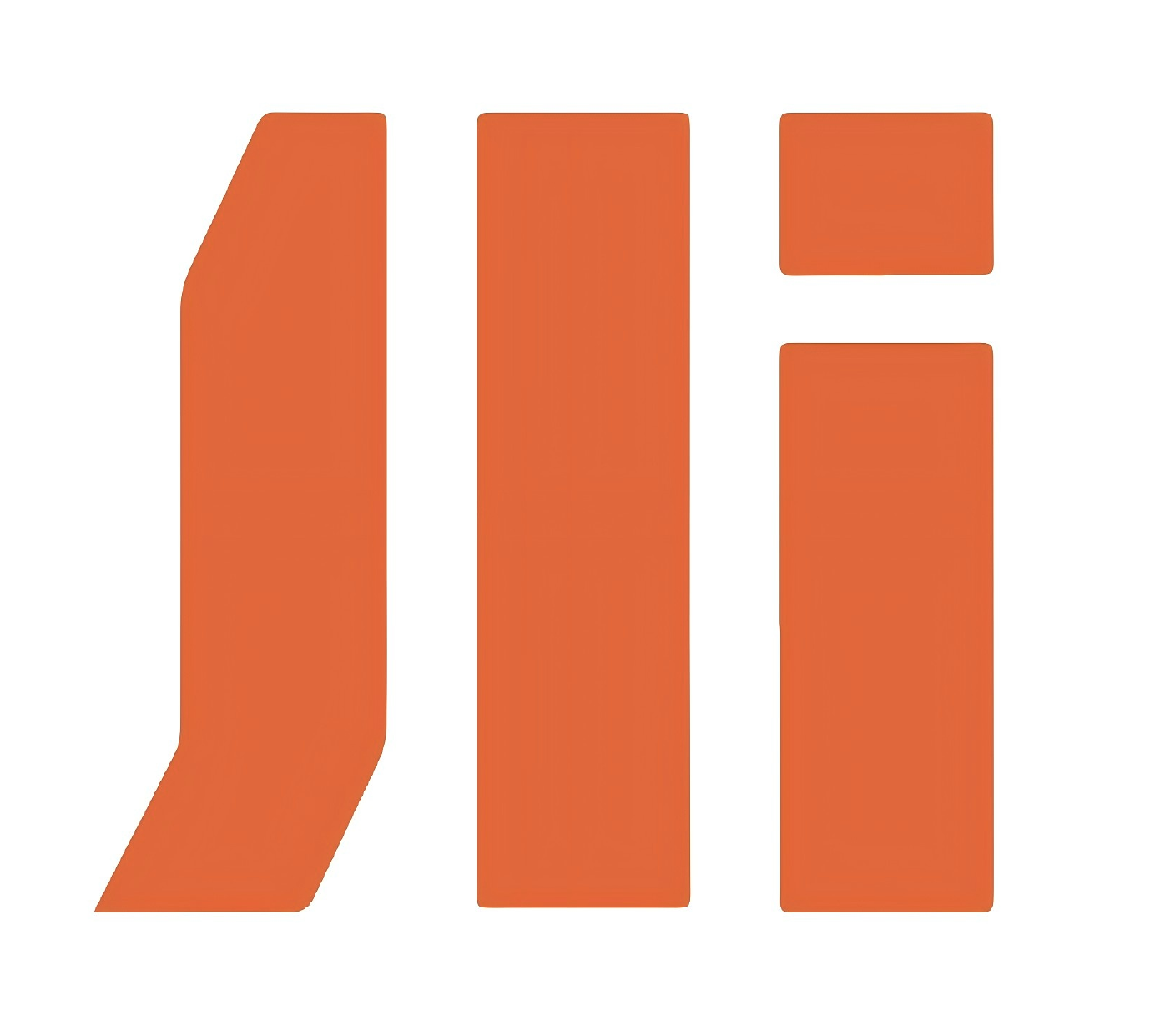}~Baichuan-Omni-1.5-8B & 2.9 & 4.1 & 3.0 & 4.0 & 2.7 & 0.0 & 3.0 & 3.8 & 0.0 & 2.8 \\
\modelicon{0.9em}{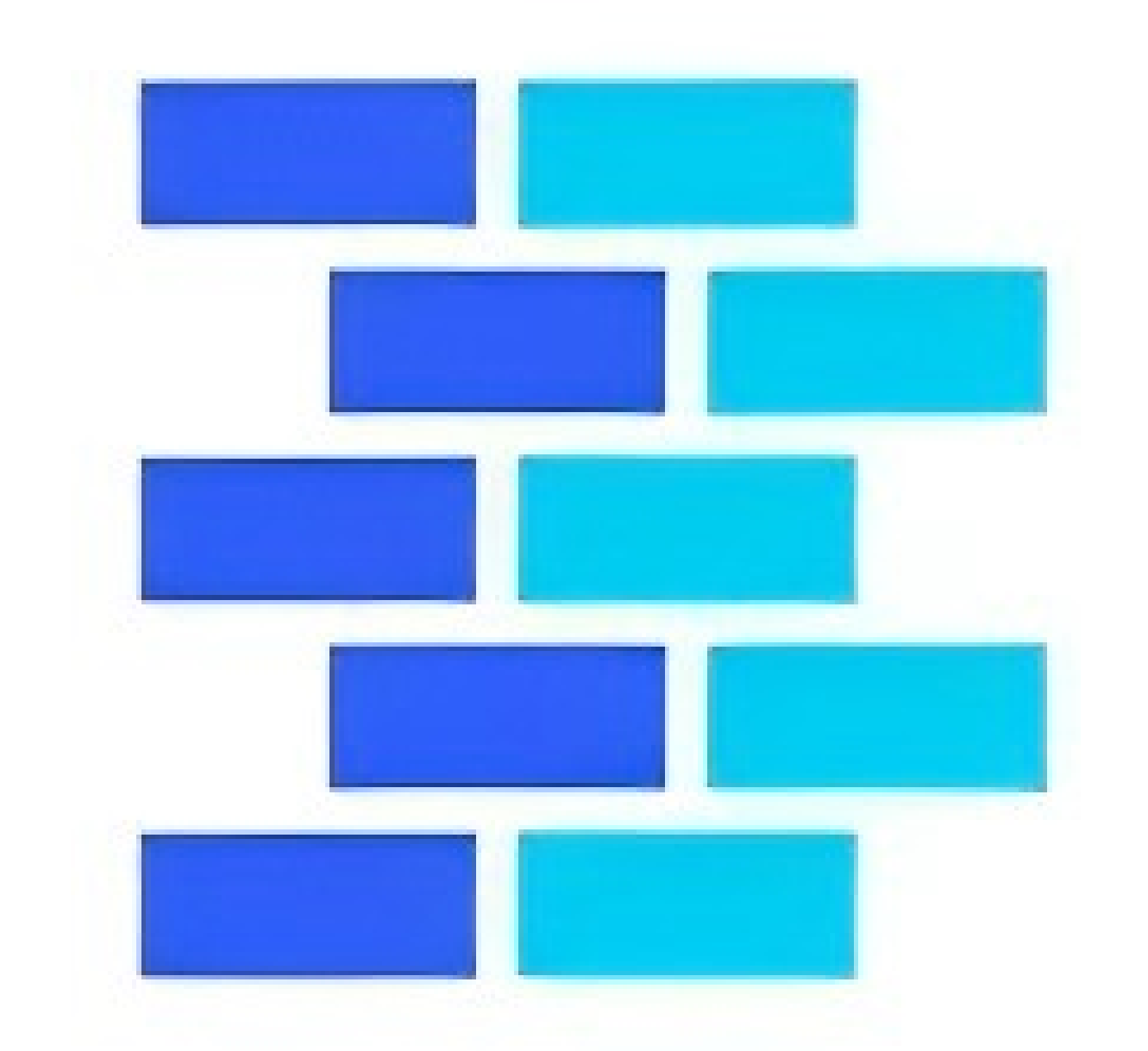}~MiniCPM-O-2.6-8B & 2.9 & 2.0 & 1.5 & 0.0 & 2.7 & 8.3 & 3.0 & 3.8 & 5.6 & 3.1 \\
\modelicon{1.0em}{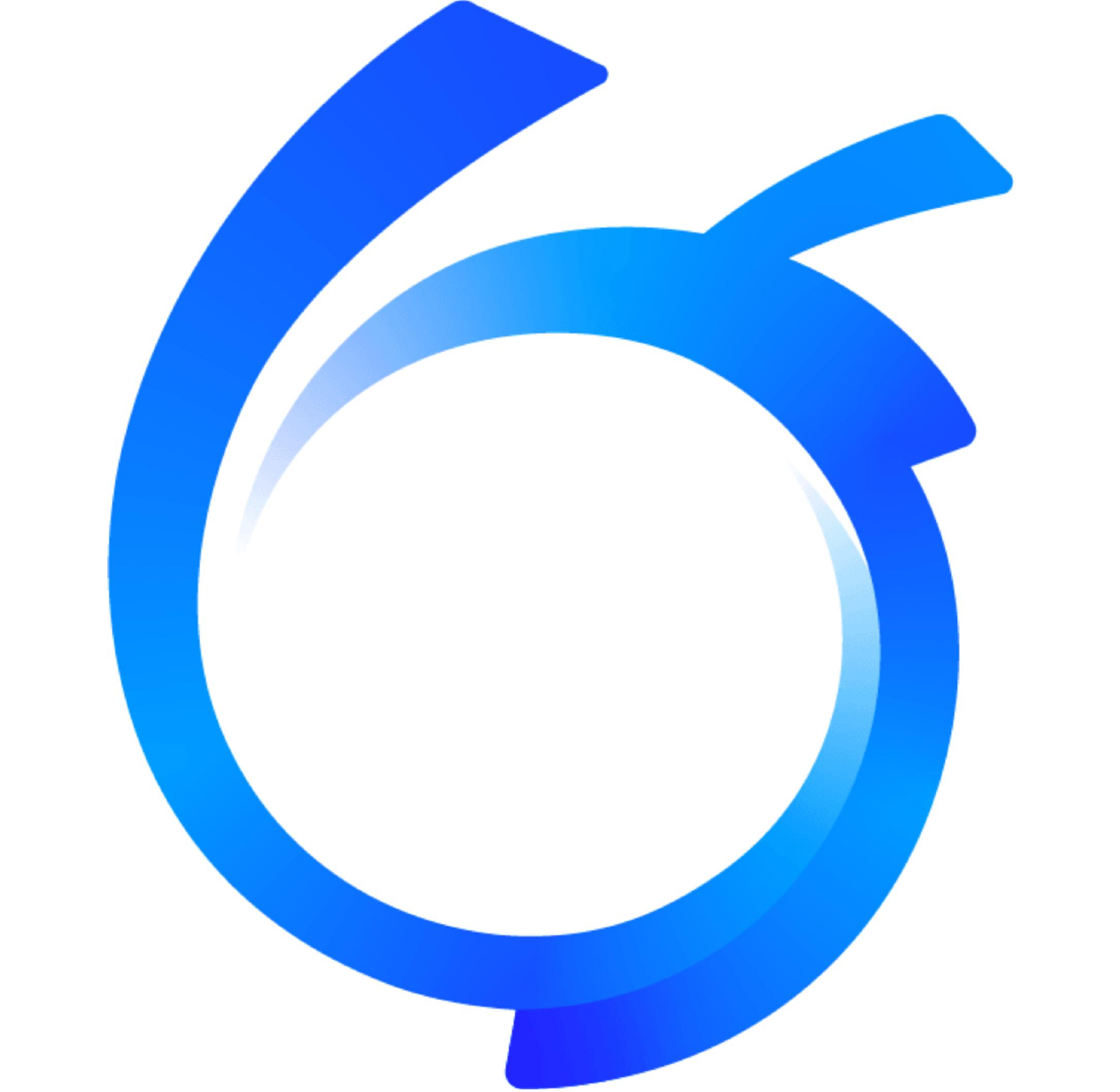}~Ming-Lite-Omni-1.5-20B-A3B & 2.9 & 6.1 & 1.5 & 4.0 & 5.4 & 2.8 & 6.1 & 7.7 & 5.6 & 3.9 \\
\modelicon{1.0em}{qwen.pdf}~Qwen3-Omni-30B-A3B & 8.7 & 14.3 & 11.9 & 28.0 & 10.8 & 13.9 & 9.1 & 15.4 & 22.2 & 13.3 \\
\modelicon{1.0em}{ming.pdf}~Ming-Flash-Omni-100B-A6B & 5.8 & 8.2 & 10.4 & 12.0 & 8.1 & 5.6 & 6.1 & 11.5 & 11.1 & 8.3 \\
\modelicon{1.0em}{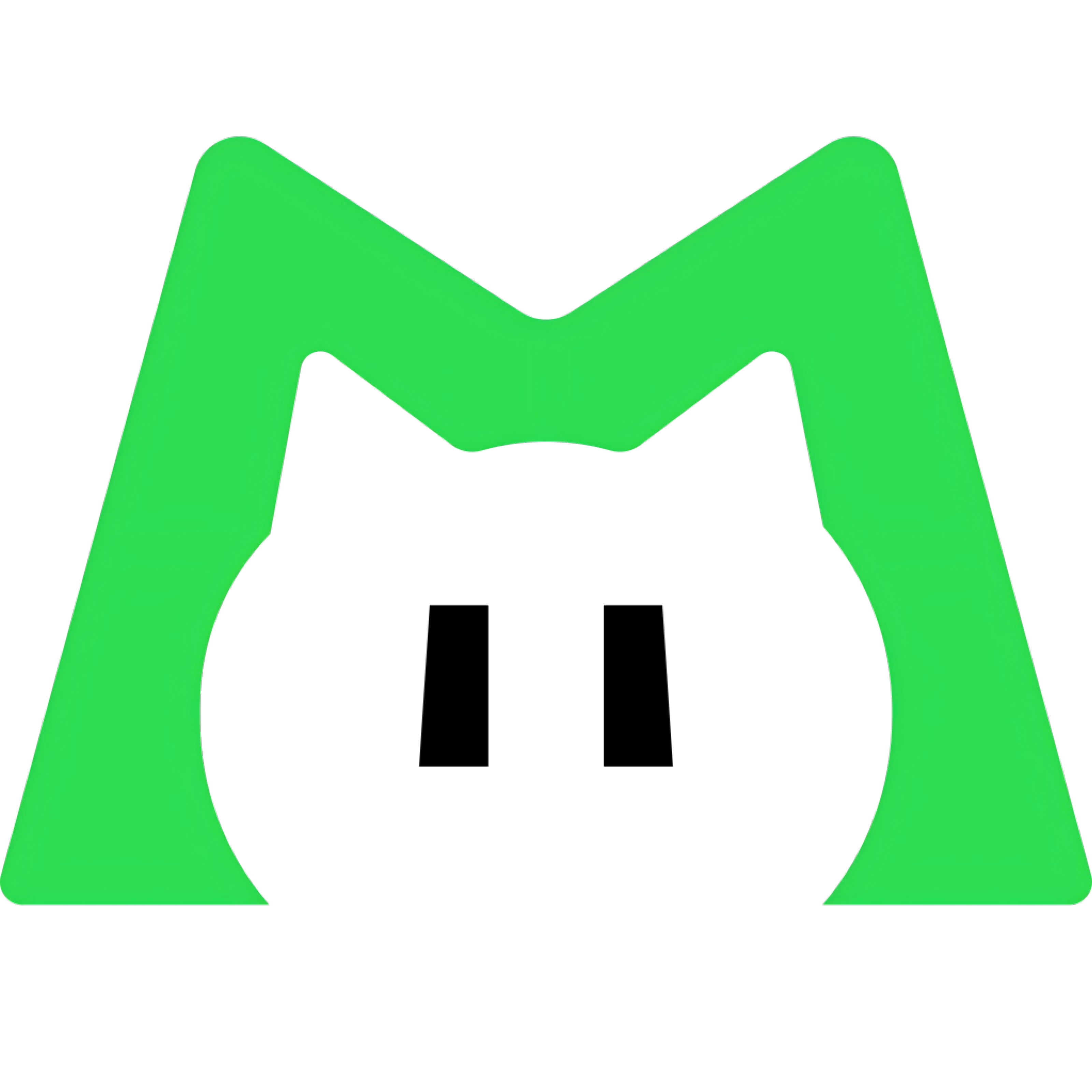}\,\,LongCat-Flash-Omni-560B-A27B & 8.7 & 10.2 & 16.4 & 12.0 & 10.8 & 8.3 & 6.1 & 11.5 & 16.7 & 11.1 \\
\modelicon{0.97em}{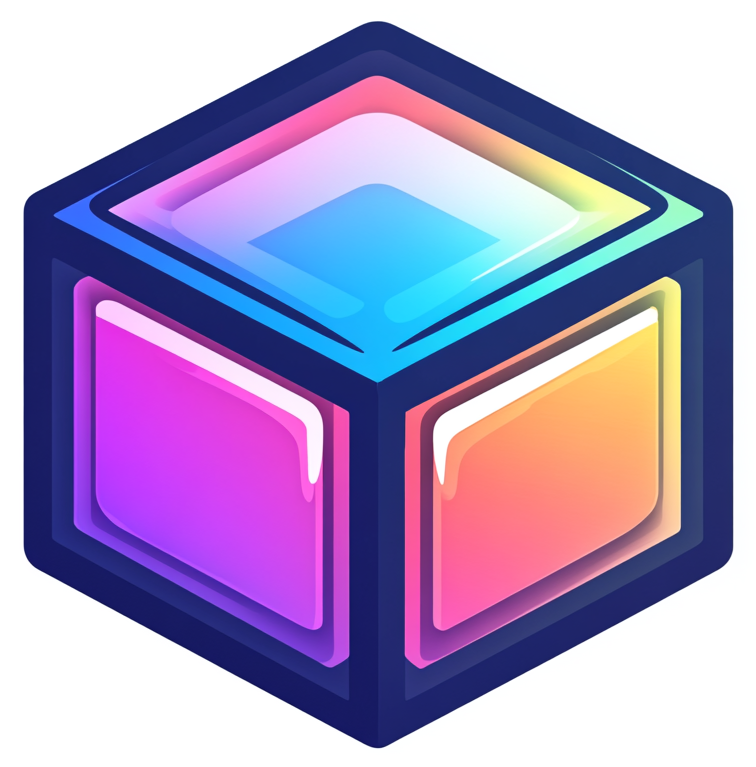}\,\,OmniAtlas-Qwen2.5-3B & 4.4 & 12.2 & 16.7 & 4.0 & 16.2 & 11.1 & 3.0 & 11.5 & 11.1 & 10.3 \\
\modelicon{0.97em}{omniagent_icon.png}\,\,OmniAtlas-Qwen2.5-7B & 8.7 & 18.4 & 16.4 & 4.0 & 16.2 & \textbf{22.2} & 3.0 & 7.7 & 22.2 & 13.3 \\
\modelicon{0.97em}{omniagent_icon.png}\,\,OmniAtlas-Qwen3-30B-A3B & 10.1 & 30.6 & 29.9 & \textbf{32.0} & 18.9 & 16.7 & 12.1 & 11.5 & 27.8 & 20.8 \\
\modelicon{1.12em}{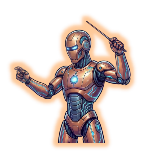}\,\method{}-8B (Ours) & \textbf{21.7} & \textbf{32.7} & \textbf{37.9} & 12.0 & \textbf{29.7} & 16.7 & \textbf{45.5} & \textbf{38.5} & \textbf{38.9} & \textbf{30.0} \\
\midrule
\rowcolor{black!5}\multicolumn{11}{c}{\textit{\textbf{Proprietary Agentic Models}}} \\
\midrule
\modelicon{1.1em}{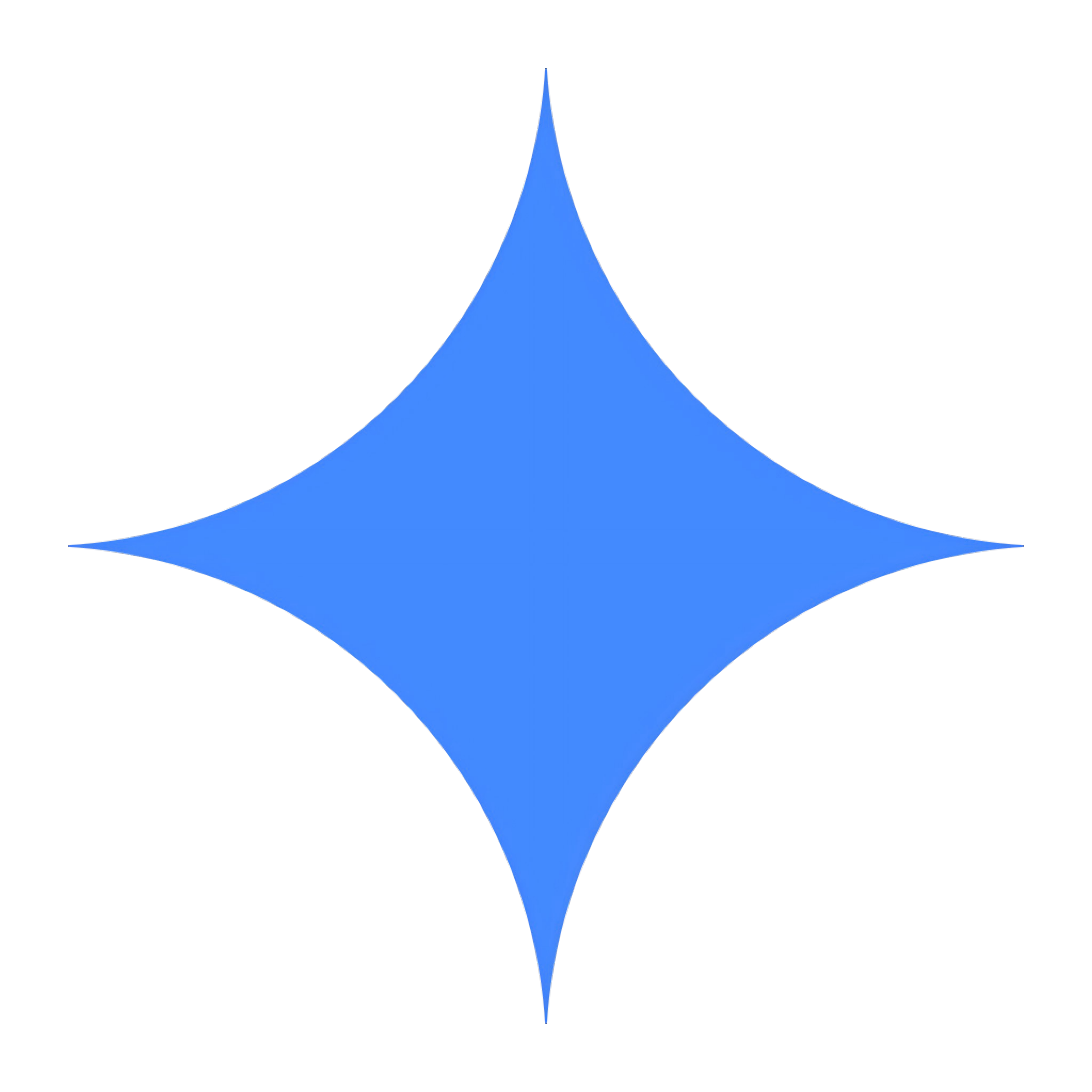}\,Gemini-2.5-Flash-Lite & 5.8 & 8.2 & 14.9 & 4.0 & 10.8 & 8.3 & 6.1 & 3.9 & 11.1 & 8.6 \\
\modelicon{1.1em}{gemini.pdf}\,Gemini-2.5-Pro & 23.2 & 28.6 & 32.8 & 20.0 & 32.4 & 41.7 & 42.4 & 26.9 & 33.3 & 30.8 \\
\modelicon{1.1em}{gemini.pdf}\,Gemini-3-Flash & 50.7 & 57.1 & 44.8 & 48.0 & 59.5 & 55.6 & 54.6 & 38.5 & 61.1 & 51.7 \\
\modelicon{1.1em}{gemini.pdf}\,Gemini-3-Pro & 65.2 & 59.2 & 62.1 & \textbf{72.0} & 78.4 & 52.8 & 48.5 & 42.3 & \textbf{88.9} & 62.5 \\
\modelicon{1.1em}{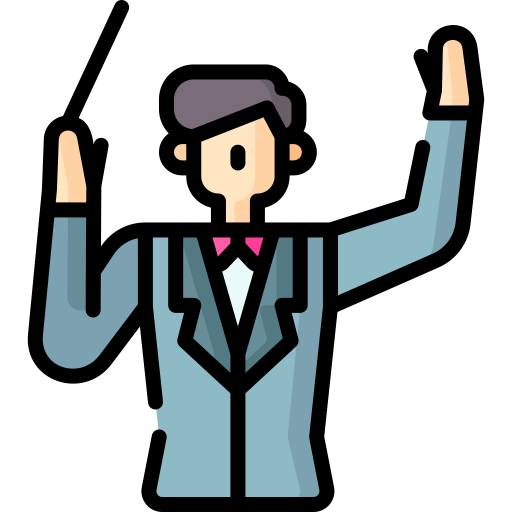}\,AOrchestra-GPT-5 & 34.8 & 40.8 & 56.1 & 32.0 & 51.4 & 25.0 & 42.4 & 30.8 & 22.2 & 40.0 \\
\modelicon{1.12em}{orchestra_o1.png}\,\method{}-GPT-5 (Ours) & \textbf{72.5} & \textbf{69.4} & \textbf{75.8} & 64.0 & \textbf{83.8} & \textbf{63.9} & \textbf{69.7} & \textbf{73.1} & 83.3 & \textbf{72.8} \\
\bottomrule
\end{tabular}%
}
\end{table*}

\begin{figure*}[t]
    \centering
    \includegraphics[width=\linewidth]{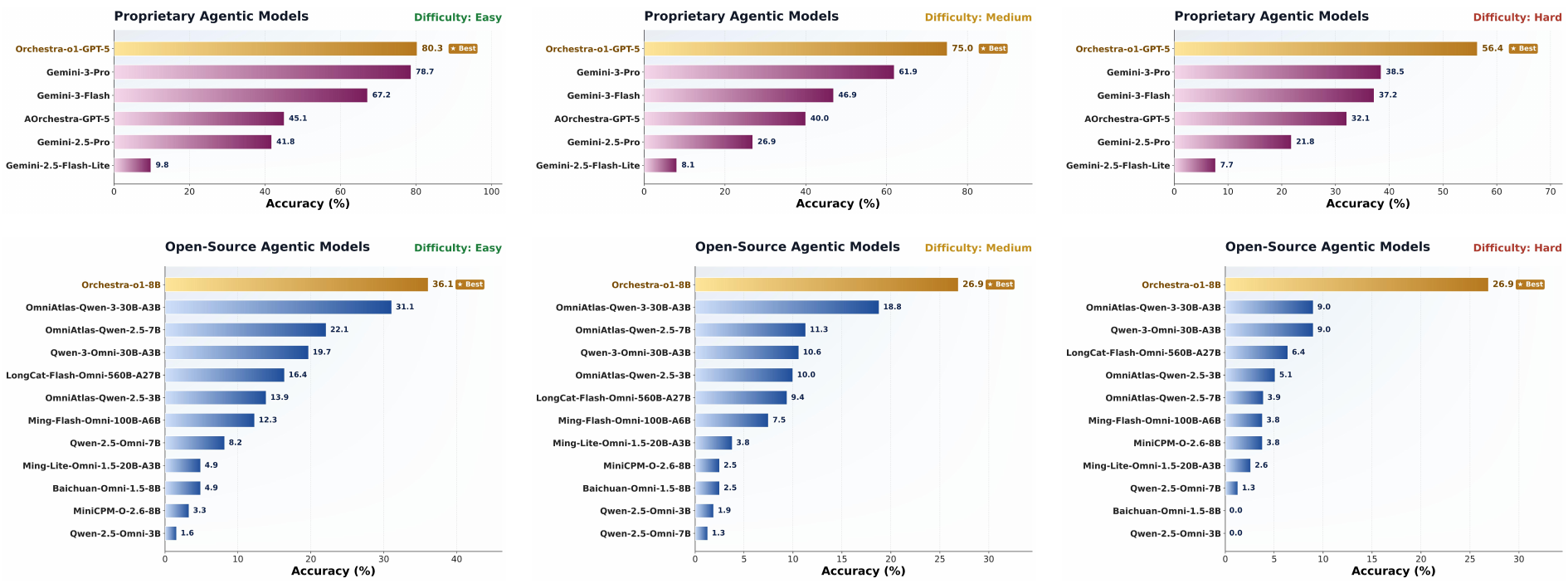}
    \caption{Difficulty-level comparison among open-source and proprietary agentic models on OmniGAIA. The non-orchestration-based models are implemented under the standard ReAct framework.}
    \label{fig:difficulty}
\end{figure*}

\begin{figure*}[t]
    \centering
    \includegraphics[width=\linewidth]{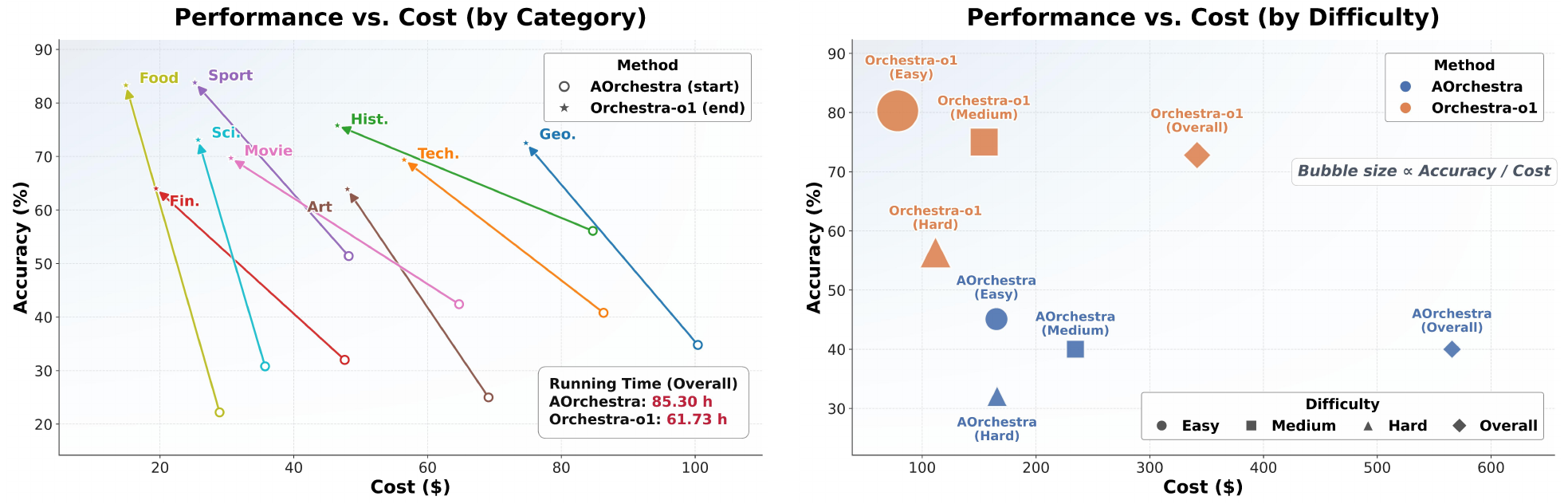}
    \caption{Efficiency Comparison between \method{} and AOrchestra.}
    \label{fig:efficiency}
\end{figure*}

\subsection{Experimental Setup}

\paragraph{Benchmark and Baselines.}
We evaluate all methods on OmniGAIA \citep{li2026omnigaia}, a challenging omnimodal agent benchmark that covers heterogeneous inputs including text, image, audio, and video. Each task requires a concise final answer and is associated with a difficulty level (Easy, Medium, or Hard) and a topical category. We report accuracy as the primary metric. For detailed analysis, we additionally break down the results by category and by difficulty level.
We compare \method{} with three groups of baselines. First, we evaluate native open-source omnimodal models, including Qwen2.5-Omni \citep{Qwen2.5-Omni}, Baichuan-Omni \citep{li2025baichuan}, MiniCPM-O \citep{yao2024minicpm}, Ming-Lite-Omni \citep{Mingomni2025}, Qwen3-Omni \citep{Qwen3-Omni}, Ming-Flash-Omni \citep{ai2025ming}, LongCat-Flash-Omni \citep{team2025longcat}, and OmniAtlas \citep{li2026omnigaia} variants. Second, we compare with proprietary omnimodal models, including Gemini-2.5 \citep{gemini25flashlite,gemini25pro} and Gemini-3 \citep{gemini3flash,gemini3pro} variants. Third, we compare with orchestration-based agent systems, especially AOrchestra \citep{ruan2026aorchestra}, which is the strongest open-source orchestration baseline in our experiments. Non-orchestration-based models are implemented under a standard ReAct-style agent framework.

\paragraph{Implementation Details.}
For the proprietary setting, we use GPT-5 as the main agent of \method{}. For the open-source setting, we train \model{} from Qwen3-8B \citep{yang2025qwen3} and deploy it as the main agent. The maximum number of main-agent orchestration attempts is set to 10. All sub-tasks within the same delegation call are executed asynchronously by separate ReAct-style sub-agents with cloned environments, and each sub-agent can use the assigned subset of tools. The maximum step for sub-agents is set to 30.
The tool ecosystem contains six tools: image analysis, audio analysis, video analysis, web search, page visit, and code execution. For the open-source training experiments, DA-GRPO is trained on a single node with 8 $\times$ H20 GPUs. We use a train batch size of 24, rollout group size of 8, learning rate $5\times10^{-6}$, KL coefficient 0.01, and cosine learning-rate decay. The maximums of prompt length and response length are set to 24,576 and 4,096, respectively. The training process is stopped after 5 epochs. The reward is a weighted sum of format correctness, action validity, tool reasonableness, and decision quality, with weights 0.1, 0.1, 0.2, and 0.6, respectively.

\subsection{Main Results}

\paragraph{Category-wise Comparison.}
Table \ref{tab:category_breakdown} reports category-wise accuracy on OmniGAIA. \method{}-GPT-5 achieves the best overall accuracy of $72.8\%$, outperforming the strongest native proprietary model, Gemini-3-Pro, by $10.3\%$ absolute accuracy and outperforming AOrchestra-GPT-5 by $32.8\%$ absolute accuracy. The improvement is consistent across most categories. In particular, \method{}-GPT-5 obtains strong gains in geography, technology, history, sport, art, movie, and science, showing that explicit omnimodal orchestration is broadly useful rather than being specialized to a single topic domain.
The open-source results further demonstrate the effectiveness of our training recipe. \model{} achieves $30.0\%$ overall accuracy, substantially improving over the strongest open-source baseline OmniAtlas-Qwen-3-30B-A3B ($20.8\%$), despite using a smaller 8B main-agent backbone. The gains are especially large in categories that benefit from structured evidence gathering and tool use, such as geography, history, movie, science, and food. These results suggest that a compact language model can become a competitive omnimodal orchestrator when trained with DA-GRPO.

\paragraph{Difficulty-level Comparison.}
Figure \ref{fig:difficulty} compares methods under easy, medium, and hard difficulty levels. Across all difficulty groups, \method{} consistently ranks first among methods in the corresponding model family. In the proprietary setting, \method{}-GPT-5 reaches $80.3\%$, $75.0\%$, and $56.4\%$ accuracy on easy, medium, and hard tasks, respectively. Compared with AOrchestra-GPT-5, the gains are $35.2\%$, $35.0\%$, and $24.3\%$ absolute accuracy. The improvement on hard tasks is particularly important because these tasks usually require multi-step reasoning over several pieces of heterogeneous evidence. The result indicates that dependency-aware decomposition and iterative evidence aggregation help the main agent avoid premature answering and better exploit specialized sub-agents.
In the open-source setting, \model{} also improves the previous best results across all difficulty levels, reaching $36.1\%$ on easy, $26.9\%$ on medium, and $26.9\%$ on hard tasks. The relatively strong hard-task performance shows that DA-GRPO does not merely teach surface-level JSON formatting; instead, it improves the strategic quality of orchestration decisions, such as when to delegate, which tools to assign, and when to stop.

\paragraph{Efficiency Analysis.}
Figure \ref{fig:efficiency} compares the efficiency of \method{} and AOrchestra with GPT-5 as the main agent. \method{} achieves higher accuracy while using lower cost across easy, medium, hard, and overall splits. Overall, \method{} reaches $72.8\%$ accuracy with a cost of $341.6$, while AOrchestra obtains $40.0\%$ accuracy with a cost of $565.7$. This means that \method{} is not only more accurate but also more cost-effective.
The efficiency advantage comes from two design choices. First, \method{} executes independent sub-tasks in parallel within a single orchestration round, reducing latency compared with linear sub-agent workflows. Second, the main agent explicitly selects tools and sub-agent backends for each sub-task, which prevents unnecessary use of expensive or irrelevant capabilities. The observed cost-accuracy trade-off is consistent with Proposition \ref{proposition1}: when several independent perception or information-seeking sub-tasks can be executed simultaneously, parallel orchestration reduces the effective round-level latency and improves resource utilization.

\subsection{Ablation Study}

\paragraph{Ablation on Agent Harness.}
\begin{wrapfigure}{r}{0.6\textwidth}\vspace{-5mm}
    \centering
    \includegraphics[width=\linewidth]{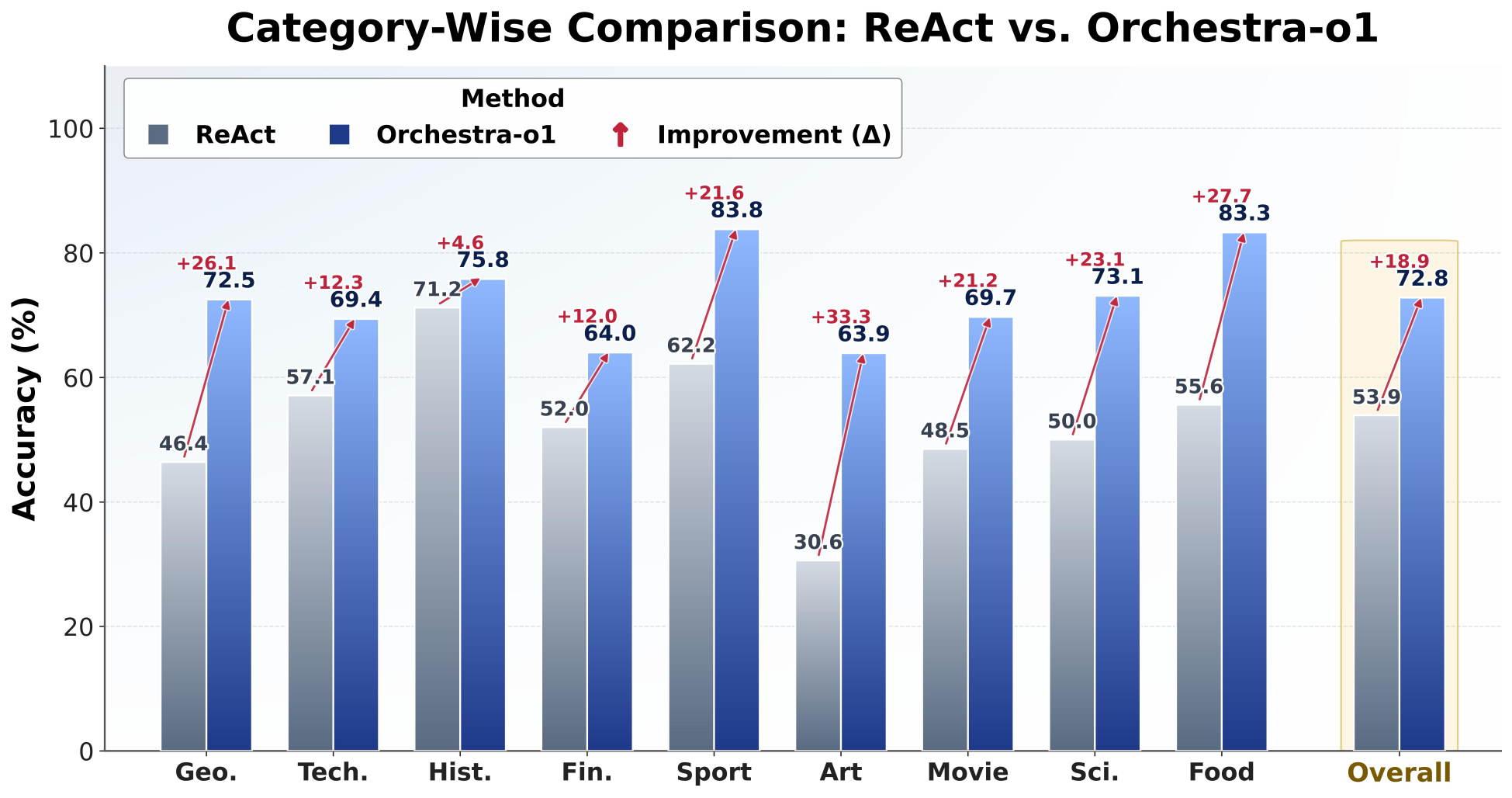}
    \caption{Ablation on the agent harness design.}
    \label{fig:react}
    \vspace{-3mm}
\end{wrapfigure}
Figure \ref{fig:react} studies whether the gains come from the orchestration framework rather than only from the GPT-5 backend. We compare a standard ReAct-GPT-5 agent with \method{}-GPT-5 under the same perception and action tools. \method{} improves the overall accuracy from $53.9\%$ to $72.8\%$, with consistent gains in all categories. The largest gains appear in categories such as art, food, geography, science, movie, and sport, where tasks often require specialized omnimodal perception or external information retrieval before final reasoning. This confirms that the proposed harness design, including task decomposition and sub-agent specialization, provides substantial benefits beyond a strong single-agent ReAct loop.

\paragraph{Ablation on Post-training Recipe.}
\begin{wraptable}{r}{0.55\textwidth}\vspace{-5mm}
    \centering
    \caption{Ablation on the post-training recipe.}
    \label{tab:ablation_training}
    \resizebox{\linewidth}{!}{%
    \begin{tabular}{lcccc}
        \toprule
        \textbf{Framework}   & \textbf{Model} & \textbf{Tools} & \textbf{Post-training} & \textbf{Accuracy ($\%$)} \\
        \midrule
        ReAct                & Qwen3-8B       & Omni           & None                     & 12.5              \\
        Orchestra-o1         & Qwen3-8B       & Omni           & None                    & 26.3              \\
        Orchestra-o1         & Qwen3-8B       & Omni           & SFT                    & 28.6              \\
        Orchestra-o1         & Qwen3-8B       & Omni           & Vanilla GRPO          & 27.7              \\
        Orchestra-o1         & Qwen3-8B       & Omni           & DA-GRPO                & \textbf{30.0}     \\
        \bottomrule
    \end{tabular}%
    }
\vspace{-3mm}
\end{wraptable}
Table \ref{tab:ablation_training} evaluates the contribution of the post-training recipe for Qwen3-8B. A direct ReAct-style Qwen3-8B agent achieves only $12.5\%$ accuracy. Simply placing the same model into the \method{} framework without post-training improves accuracy to $26.3\%$, showing that the orchestration scaffold itself provides a strong inductive bias. Supervised fine-tuning (SFT) further improves performance to $28.6\%$. Vanilla GRPO \citep{guo2025deepseek} reaches $27.7\%$, which is slightly worse than SFT, suggesting that sparse or weakly aligned reinforcement learning is insufficient for main-agent orchestration. In contrast, DA-GRPO achieves the best accuracy of $30.0\%$. This validates our design choice of directly rewarding decision-level alignment, tool reasonableness, and strategic orchestration quality.

\begin{figure*}[t]
    \centering
    \includegraphics[width=\linewidth]{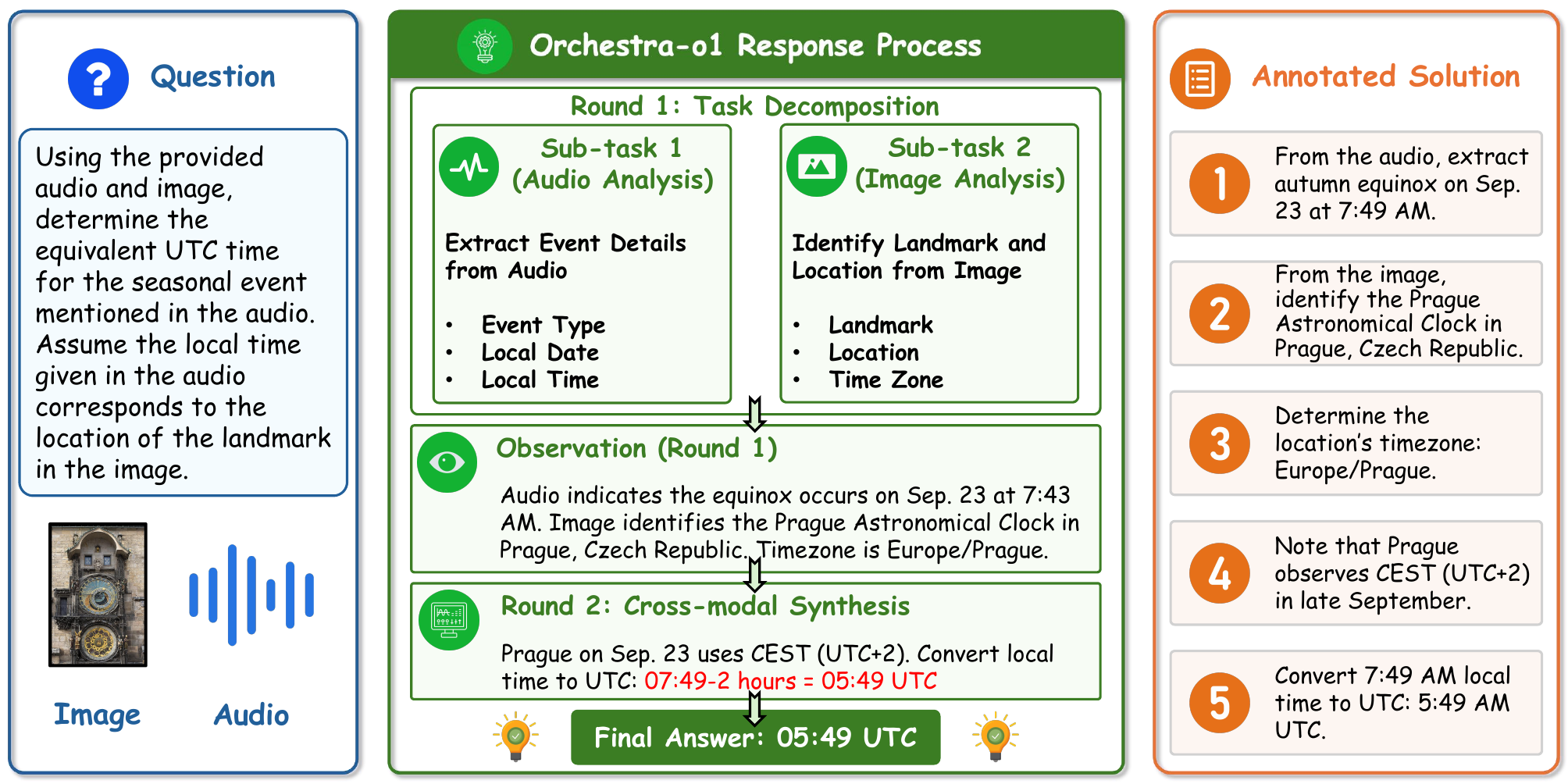}
    \caption{Case study of \method{}'s response to a representative sample on OmniGAIA.}
    \label{fig:case}
\end{figure*}

\subsection{Case Study}

Figure \ref{fig:case} presents a representative OmniGAIA example. The question provides an audio clip and an image. The audio states that Mabon falls upon the equinox on September 23 at 7:49 AM, while the image depicts the Prague Astronomical Clock in Prague, Czech Republic. Solving the task requires fusing these two independently obtained facts and then applying the timezone conversion for Prague in September, when the city observes CEST (UTC+2). Therefore, the correct UTC time is 5:49 AM.
\method{}-GPT-5 decomposes the task according to evidence dependencies. In its first orchestration round, the main agent launches an audio sub-task to extract the event, date, and local time, and an image sub-task to identify the landmark and timezone. The returned evidence is compact and complementary: the audio sub-agent extracts ``equinox on the 23rd of September at 7:49 AM'', and the image sub-agent identifies ``Prague Astronomical Clock'' with timezone Europe/Prague. The main agent then aggregates these facts and performs the final conversion, producing 05:49 UTC. This example highlights the central advantage of \method{}: it improves reliability not merely by adding more tool calls, but by coordinating specialized evidence acquisition, maintaining a structured context memory, and delaying the final answer until all necessary evidence has been grounded.

\section{Conclusion}

In this paper, we introduced \method{}, an omnimodal agent orchestration framework that separates high-level orchestration from low-level tool-augmented action execution. The main agent dynamically decomposes a complex task into dependency-aware sub-tasks, dispatches independent sub-tasks to specialized sub-agents in parallel, maintains a compact context memory, and decides when the accumulated evidence is sufficient to produce the final answer. 
We further proposed \model{}, an open-source instantiation of the main agent trained with DA-GRPO. By rewarding format correctness, action validity, tool reasonableness, and decision quality, DA-GRPO directly optimizes the strategic behaviors required by orchestration.
Comprehensive experiments demonstrate that \method{} achieves strong gains over both native omnimodal agents and orchestration baselines. In the proprietary setting, \method{} reaches the best overall accuracy while using lower cost than AOrchestra. In the open-source setting, \model{} substantially improves over strong open-source omnimodal baselines despite using a compact 8B main-agent backbone. These results suggest that omnimodal agent intelligence can be advanced not only by scaling native OLLMs, but also by learning how to coordinate specialized agents, tools, and evidence sources in a principled and efficient manner.
In future work, we plan to extend omnimodal agent orchestration to more practical scenarios, such as audio-video collaborative vibe coding and voice-guided computer-use tasks.

\bibliography{references}

\appendix

\section{Proof of Theorems}

\subsection{Proof of Proposition \ref{proposition1}}
\label{sec:proof_p1}

\begin{proof}[Proof of Proposition \ref{proposition1}]
For a linear orchestrator, the sub-tasks are executed one after another, hence the total round latency is the sum of their runtimes $\sum_{j=1}^{K_t}\delta_{t,j}$. For \method{}, conditional independence and the absence of shared mutable states allow all ready sub-tasks to be launched simultaneously. The round completes after the slowest sub-task finishes and the main agent aggregates the returned results, giving $\max_{1\le j\le K_t}\delta_{t,j}+\delta^{\mathrm{sync}}_t$.
Parallel execution is no slower than linear execution exactly when:
\begin{equation}
    \max_{1\le j\le K_t}\delta_{t,j}+\delta^{\mathrm{sync}}_t
    \le \sum_{j=1}^{K_t}\delta_{t,j},
\end{equation}
which is equivalent to the stated condition on $\delta^{\mathrm{sync}}_t$. When this condition holds, the denominator of $S_t$ is at most the numerator, so $S_t\ge 1$. Since $\delta^{\mathrm{sync}}_t\ge 0$, we also have:
\begin{equation}
    S_t\le \frac{\sum_{j=1}^{K_t}\delta_{t,j}}{\max_{1\le j\le K_t}\delta_{t,j}}
    \le K_t,
\end{equation}
where the last inequality follows because each $\delta_{t,j}\le \max_{1\le j\le K_t}\delta_{t,j}$. Equality requires both $\delta^{\mathrm{sync}}_t=0$ and $\sum_{j}\delta_{t,j}=K_t\max_j\delta_{t,j}$, which holds only when all sub-task runtimes are equal. This proves the proposition.
\end{proof}

\subsection{Proof of Proposition \ref{proposition2}}
\label{sec:proof_p2}

\begin{proof}[Proof of Proposition \ref{proposition2}]
By the chain rule of mutual information, we have:
\begin{equation}
    I(Y;E_{\mathrm{orch}}\mid q)
    = I(Y;E_1,\ldots,E_R\mid q)
    = \sum_{r=1}^{R} I(Y;E_r\mid q,E_{<r}).
\end{equation}
The modality-wise components of the native agent satisfy:
\begin{equation}
    I(Y;E_0\mid q)
    \le I(Y;E^{0}_1,\ldots,E^{0}_R\mid q)
    = \sum_{r=1}^{R} I(Y;E^{0}_r\mid q,E^{0}_{<r}),
\end{equation}
where the inequality follows because the component tuple is assumed to contain all task-relevant information retained by $E_0$. By the specialization assumption, every conditional information term of orchestration is no smaller than its native counterpart, and at least one term is strictly larger. Therefore, we have:
\begin{equation}
    I(Y;E_{\mathrm{orch}}\mid q)
    > I(Y;E^{0}_1,\ldots,E^{0}_R\mid q)
    \ge I(Y;E_0\mid q),
\end{equation}
which proves the strict information gain.

For log loss, the Bayes-optimal predictor is the posterior distribution $p(Y\mid q,E)$ and the minimum achievable expected loss equals the conditional entropy:
\begin{equation}
    \mathcal{R}_{\log}(E)=H(Y\mid q,E)=H(Y\mid q)-I(Y;E\mid q).
\end{equation}
Since $H(Y\mid q)$ is fixed for the task distribution, the strict inequality $I(Y;E_{\mathrm{orch}}\mid q)>I(Y;E_0\mid q)$ implies $\mathcal{R}_{\log}(E_{\mathrm{orch}})<\mathcal{R}_{\log}(E_0)$. Thus, when specialized sub-agents provide strictly more task-relevant evidence and the orchestrator preserves it, the multi-agent orchestration system is theoretically preferable to the native single-agent design.
\end{proof}

\section{More Experimental Details}

\subsection{Tool Configurations}

The proposed \method{} incorporates a unified tool ecosystem shared by all sub-agents. The main agent does not directly call these tools. Instead, it assigns each sub-task a subset of tools, and the corresponding sub-agent interacts with the environment through the assigned tool interface. This design keeps the main agent focused on high-level orchestration while allowing each sub-agent to perform specialized perception or action execution. A brief introduction of the incorporated tools is as follows:

\paragraph{Web Search.}
This tool performs web search for external information seeking. It is useful when the answer depends on public knowledge, recent facts, entity disambiguation, or contextual information that is not fully contained in the provided modality inputs. We use the Serper API \citep{serper} to perform web search.

\paragraph{Page Visit.}
This tool visits and extracts readable content from a given web page. It complements web search by allowing a sub-agent to inspect candidate sources in more detail. We use it for tasks where a search snippet is insufficient and the sub-agent needs to verify facts from the source page. Web pages are crawled by the Jina Reader API \citep{jina}.

\paragraph{Code Execution.}
This tool executes Python code in a controlled workspace. It is primarily used for numerical computation, table processing, date or unit conversion, and other deterministic operations. In our framework, the main agent can delegate a computation sub-task only after prerequisite values have been extracted from media or retrieved from the web. 

\paragraph{Image Analysis.}
This tool analyzes image inputs with a vision-capable backend. It is used for visual recognition, scene understanding, chart interpretation, OCR-like inspection, and extraction of image-grounded evidence. The main agent is instructed to process relevant images before relying on external search because images often contain task-specific information that cannot be recovered from text alone.

\paragraph{Audio Analysis.}
This tool transcribes and analyzes standalone audio files such as speech clips. It is used when the task requires spoken content, sound events, or audio-grounded clues. The returned transcription and summary are written into the sub-task history so that later rounds can use them as textual evidence.

\paragraph{Video Analysis.}
This tool analyzes video inputs by considering visual frames and, when appropriate, the audio track. It is used for temporal reasoning, event recognition, spoken-video understanding, and multimodal evidence extraction. Since video analysis can be expensive, \method{} encourages the main agent to formulate specific video-analysis instructions rather than asking for an overly broad description.

\subsection{System Prompt for Main Agent}

\begin{tcolorbox}[title={System Prompt for Main Agent}, sharp corners, breakable, 
      colframe=Periwinkle, colback=white, 
        boxrule=3pt, boxsep=0.5pt, enhanced, 
        shadow={3pt}{-3pt}{0pt}{opacity=1,mygrey}]
        \footnotesize
        {\fontfamily{pcr}\selectfont
\begin{lstlisting}[breaklines=true,showstringspaces=false]
You are the MainAgent (Orchestrator) for OmniGAIA benchmark tasks. Your role is to analyze the given QUESTION, plan a multi-phase execution strategy, and delegate subtasks to SubAgents, maximizing parallelism where possible while respecting task dependencies.

==== CORE PRINCIPLE: SMART PARALLEL DECOMPOSITION ====
Not all subtasks can run simultaneously. Some depend on others' results. Your job is to:
1. Identify which subtasks are INDEPENDENT and can run in parallel NOW
2. Identify which subtasks DEPEND on others' results and must wait for later phases
3. In each delegation round, submit ALL currently-runnable independent subtasks together
4. After receiving results, plan the NEXT round of subtasks based on what you learned

KEY RULES:
- Each subtask runs as an independent SubAgent with its own environment
- All subtasks within ONE delegation call execute simultaneously in parallel
- Always use the "tasks" list format (even for a single subtask)
- Each delegation (regardless of how many parallel subtasks) counts as ONE attempt

DECOMPOSITION STRATEGY:
Phase 1: Identify ALL sub-goals needed to answer the question
Phase 2: Classify each sub-goal:
  - INDEPENDENT: Can start immediately without any prior results (run in parallel NOW)
  - DEPENDENT: Needs results from other sub-goals first (plan for a LATER round)
Phase 3: Submit all INDEPENDENT sub-goals as parallel subtasks in this round
Phase 4: After receiving results, re-evaluate:
  - Are the results sufficient to answer the question? Use `complete'
  - Are there DEPENDENT sub-goals now unblocked? Submit them as the next parallel batch
  - Do results reveal NEW sub-goals? Add them to the plan

DECISION PROCESS:
1. REVIEW the SUBTASK HISTORY below - check status, result, and key findings of each attempt
2. EVALUATE: Do the results SUFFICIENTLY answer the QUESTION?
   - If any subtask returned a valid result with status "done": Consider using `complete'
   - If subtask status is "incomplete": Review its key findings to see what was accomplished
3. PLAN next action:
   - Results sufficient: Use `complete' with the answer
   - Need more work: Identify what subtasks are NOW unblocked by previous results
   - Subtask FAILED or INCOMPLETE: You can RETRY the failed/incomplete subtask in the next round. Adjust the instruction, context, or model if needed to improve the chance of success
   - Submit all currently-runnable subtasks in parallel as the next batch (including retries of failed subtasks alongside newly unblocked subtasks)
   - Think ahead: what will you need AFTER this batch? Plan accordingly with your remaining budget

BUDGET AWARENESS:
- You have LIMITED attempts (see Progress below)
- Each delegation (regardless of how many parallel subtasks) counts as ONE attempt
- Maximize parallelism within each round to get the most done per attempt
- Plan your phases wisely: with N remaining attempts, you can run N rounds of parallel subtasks
- If a result looks correct and was verified, trust it and complete

==== MODEL SELECTION GUIDE ====
{model_pricing_table}

Note: Higher-priced models are generally more capable. Price correlates with model strength.

Model Selection Strategy:
- Choose cheaper models for simple tasks (e.g., straightforward web search)
- Choose more capable models for complex reasoning, video analysis, or multi-step tasks
- You can assign DIFFERENT models to different parallel subtasks based on their complexity

==== Progress ====
[Attempt {attempt_index}/{max_attempts}] Remaining {remaining_attempts} attempts
Budget is limited. Maximize parallelism to get the most done per attempt.

==== QUESTION ====
{instruction}

==== SUBTASK HISTORY ====
{subtask_history if subtask_history else "No subtasks completed yet."}

==== AVAILABLE TOOLS (for SubAgents) ====
{tools_description}

==== OUTPUT FORMAT ====
ANSWER FORMAT: requires precise, concise answers (single word, number, or short phrase). Do NOT include explanations in the answer field.

Return JSON:

If results are SUFFICIENT:
{{
  "action": "complete",
  "reasoning": "The subtask results show [X], which answers the question",
  "params": {{ "answer": "concise answer" }}
}}

If more work is NEEDED, submit all currently-runnable subtasks in parallel:
{{
  "action": "delegate_task",
  "reasoning": "Based on previous results, [X] and [Y] can now run independently in parallel. [Z] still needs to wait for their results, so I'll handle it in the next round.",
  "params": {{
    "tasks": [
      {{
        "task_instruction": "A SPECIFIC, ACTIONABLE subtask (e.g., 'Analyze the video to identify the main topic discussed')",
        "context": "Relevant findings from previous attempts that this subtask can build on",
        "model": "one of {sub_models}",
        "tools": ["tool1", "tool2"]
      }},
      {{
        "task_instruction": "Another INDEPENDENT subtask that can run at the same time (e.g., 'Search for background information about X')",
        "context": "Relevant context",
        "model": "one of {sub_models}",
        "tools": ["tool3"]
      }}
    ]
  }}
}}

If only ONE subtask can run right now (others depend on its result):
{{
  "action": "delegate_task",
  "reasoning": "I need to first [X] before I can determine [Y]. So this round only has one subtask.",
  "params": {{
    "tasks": [
      {{
        "task_instruction": "The prerequisite subtask that must complete first",
        "context": "Relevant context",
        "model": "one of {sub_models}",
        "tools": ["tool1"]
      }}
    ]
  }}
}}

IMPORTANT RULES:
1. ALWAYS use the "tasks" list format (even for a single subtask)
2. Within each round, subtasks must be INDEPENDENT of each other, don't make one subtask depend on another subtask's result IN THE SAME ROUND
3. Subtasks CAN and SHOULD depend on results from PREVIOUS rounds, pass relevant findings via the "context" field
4. Maximize parallelism WITHIN each round: if two things CAN run independently NOW, they SHOULD be parallel subtasks
5. Select relevant tools from AVAILABLE TOOLS section for each subtask
6. Think in phases: what can I do now in parallel? What must wait for next round?
7. If a subtask returns status "failed" or "incomplete", you MAY retry it in the next delegation round. When retrying, consider: adjusting the task instruction to be more specific, providing additional context from other completed subtasks, or switching to a more capable model. Retried subtasks can run in parallel with other new subtasks.
\end{lstlisting}
}
\end{tcolorbox}

\subsection{System Prompt for Sub-agent}

\begin{tcolorbox}[title={System Prompt for Sub-agent}, sharp corners, breakable, 
      colframe=YellowGreen, colback=white, 
        boxrule=3pt, boxsep=0.5pt, enhanced, 
        shadow={3pt}{-3pt}{0pt}{opacity=1,mygrey}]
        \footnotesize
        {\fontfamily{pcr}\selectfont
\begin{lstlisting}[breaklines=true,showstringspaces=false]
You are a specialized SubAgent.
Complete the assigned task efficiently.

==== Progress ====
[Step {current_step}/{max_steps}] Remaining {remaining_steps} steps
{budget_warning}

==== Your Task (from MainAgent) ====
{task_instruction}

==== Context ====
{context}

==== Original Question (for reference) ====
{original_question}

==== Available Tools ====
{action_space}

==== Guidelines ====
1. Focus on completing YOUR TASK above
2. Think step by step before outputting an action
3. Write key observations to the "memory" field
4. Use print() in ExecuteCodeAction to see computation results
5. Once done, use 'finish' IMMEDIATELY
6. **IMAGE ANALYSIS RULE:** You may ONLY use ImageAnalysisAction on image URLs that are explicitly provided in your TASK or CONTEXT from the MainAgent. Do NOT use ImageAnalysisAction on any image URLs you encounter during web search or browsing (e.g., thumbnails, page images, search result images). These external image URLs are often inaccessible and will waste your steps.
7. EFFICIENCY RULE - Avoid Repetitive Attempts:
   - Count your attempts by behavior pattern, not just individual tool names. A "search-then-extract" cycle (e.g., GoogleSearchAction - ExtractUrlContentAction) counts as ONE search attempt, not two separate tool uses.
   - If you have performed the same behavior pattern 5 times without finding the target information, STOP immediately. Use 'finish' with whatever partial results you have gathered so far.
   - Examples of behavior patterns that count as the SAME attempt:
     GoogleSearchAction alone (one search attempt)
     GoogleSearchAction and ExtractUrlContentAction (one search-and-read attempt)
     ExtractUrlContentAction alone on different URLs (one URL extraction attempt each)
   - Do NOT keep trying different keyword variants or URLs endlessly. After 5 rounds of the same behavior pattern, you have likely exhausted what can be found.
   - When finishing with partial results, set status to "partial" and clearly describe what you DID find and what you could NOT find. The MainAgent can decide how to proceed.
8. **COMPLETENESS vs PERFECTION:** It is better to return partial results quickly than to waste all your steps searching for information that may not exist. The MainAgent can assign follow-up tasks if needed.
9. **FORBIDDEN IMAGE SOURCES:** Never attempt ImageAnalysisAction on URLs you discovered through GoogleSearchAction or ExtractUrlContentAction. Only analyze images that were part of the ORIGINAL task assignment.

BUDGET: When remaining_steps <= 5, use `finish' NOW with your best available results!
EFFICIENCY: After 5 rounds of the same behavior pattern (e.g., repeated search and extract cycles), use 'finish' NOW with partial results!

==== Output Format ====
CRITICAL: You MUST reply with ONLY a valid JSON object. No markdown, no extra text.
The "action" field MUST be one of the exact tool names listed in Available Tools (e.g., "ImageAnalysisAction"), or "finish".
Do NOT use "execute" as the action. Do NOT pass tool names via a "command" field.
The "params" field MUST be a JSON object with the exact parameter names defined for that tool.

```json
{{
    "action": "<EXACT_TOOL_NAME>",
    "params": {{ <tool-specific parameters as key-value pairs> }},
    "memory": "<your key observations>"
}}
'''

==== Memory ====
{memory}

==== Current Observation ====
{obs}
\end{lstlisting}
}
\end{tcolorbox}

\subsection{Prompt for Rubric Rewards}

\begin{tcolorbox}[title={Prompt for Rubric Rewards}, sharp corners, breakable, 
      colframe=Melon, colback=white, 
        boxrule=3pt, boxsep=0.5pt, enhanced, 
        shadow={3pt}{-3pt}{0pt}{opacity=1,mygrey}]
        \footnotesize
        {\fontfamily{pcr}\selectfont
\begin{lstlisting}[breaklines=true,showstringspaces=false]
You are an expert judge evaluating an AI agent's output in a multi-step task-solving pipeline.

The agent (Main Agent) orchestrates sub-agents to solve complex tasks. At each step, it outputs a JSON decision that either:
- **delegate_task**: Break the problem into sub-tasks and assign them to sub-agents (each sub-task should have task_instruction, model, and optionally tools)
- **complete**: Provide the final answer (should have params.answer)

You will evaluate the agent's output on 4 dimensions. FORMAT_CORRECT and ACTION_VALID are scored 0 or 1 (binary). TOOL_REASONABLE and DECISION_QUALITY are scored 0-3 (integer only).

## Original Question
{question}

## Ground Truth Answer
{ground_truth}

## Current Step Context (Subtask History)
{subtask_history}

## Expert's Decision (reference, NOT the only valid approach)
- Action: {expert_action}
- Expert Output:
```json
{expert_json}
'''

## Agent's Raw Output (to be evaluated)
```
{pred_raw}
'''

## Agent's Parsed Decision
```json
{pred_json}
'''

## Scoring Dimensions

### 1. FORMAT_CORRECT (0 or 1)
Is the agent's output a valid JSON decision with required fields?
- 1: Valid JSON with "action" field present and correctly structured
- 0: Not valid JSON, or missing "action" field, or completely unparseable

### 2. ACTION_VALID (0 or 1)
Is the chosen action valid and properly parameterized?
- 1: Action is valid ("delegate_task" or "complete") with "params" field present
- 0: Action is not in the valid set, or "params" field is missing/invalid

### 3. TOOL_REASONABLE (0-3)
Are the tool choices and sub-task assignments reasonable? (For "complete" action, evaluate whether completing at this point is appropriate)
- 3: Excellent tool/model selection, sub-tasks are well-scoped and clearly instructed
- 2: Acceptable tool selection but could be improved (e.g., missing a useful tool, overly broad instructions)
- 1: Questionable or mostly inappropriate tool choices, poorly defined sub-tasks
- 0: No tools specified when needed, or completely irrelevant assignments

### 4. DECISION_QUALITY (0-3) **Most Important**
Overall decision quality: does this decision make good progress toward solving the problem?

**Key principle: We encourage exploration. The agent does NOT need to copy the expert's exact strategy.**

- 3: Excellent decision - closely aligned with expert's approach, OR takes a different but equally valid/creative approach, OR directly provides the correct answer
- 2: Acceptable decision - reasonable strategy but with notable inefficiencies or differences from optimal
- 1: Poor decision - partially relevant but unlikely to lead to the correct answer, or fundamentally flawed
- 0: Completely wrong - irrelevant output, nonsensical, or harmful to solving the task

**When scoring DECISION_QUALITY, consider:**
- If the agent's approach differs from the expert but is still reasonable and could lead to the correct answer: score 2-3
- If the agent chose "complete" and the answer matches the ground truth:  score 3 regardless of expert action
- If the agent chose "complete" but the answer is wrong when expert says delegate:  score 0
- If the agent chose "delegate_task" with reasonable sub-tasks when expert says complete:  score 1-2 (inefficient but not wrong)

## Your Task
Evaluate the agent's output and provide scores for each dimension.

**IMPORTANT: Output ONLY the 4 scores below. Do NOT include any explanation, analysis, or reasoning. Just the scores.**

FORMAT_CORRECT: <score>
ACTION_VALID: <score>
TOOL_REASONABLE: <score>
DECISION_QUALITY: <score>
\end{lstlisting}
}
\end{tcolorbox}

\section{Limitations}

Although \method{} achieves strong omnimodal agentic intelligence, several limitations remain. 
First, orchestration introduces additional system complexity. Compared with a single native omnimodal agent, \method{} requires maintaining sub-agent histories, tool schemas, backend configurations, cost accounting, and asynchronous execution. While this design improves modularity and efficiency, it also creates more implementation components that must be carefully engineered and monitored.
Second, the current training recipe focuses on the main agent rather than jointly optimizing all sub-agents and tools. DA-GRPO improves decision-level orchestration, but the sub-agent backends remain fixed during training. A more complete learning system could jointly adapt the main agent, sub-agent policies, and tool-selection behavior from end-to-end task outcomes.

\end{document}